\documentclass[letterpaper,10pt]{article}
\usepackage{geometry}
\usepackage{amsmath,amsthm}

\usepackage[american]{babel}
\usepackage[utf8]{inputenc}
\usepackage{times}
\usepackage{lmodern}
\usepackage{microtype}

\usepackage{amsmath}
\usepackage{amsthm}
\usepackage{amssymb}
\usepackage{amsfonts}
\usepackage{mathtools}
\usepackage{bm}
\usepackage{bbm}
\usepackage{dsfont}
\usepackage{mathrsfs}
\usepackage{empheq}

\usepackage{graphicx}
\usepackage[space]{grffile}
\usepackage[small]{caption}
\usepackage{subcaption}
\usepackage{booktabs}
\usepackage{tabularx}
\usepackage{makecell}
\usepackage{adjustbox}
\usepackage{nicefrac}
\usepackage{float}
\usepackage{placeins}

\usepackage{xcolor}
\usepackage{color}
\usepackage{soul}
\usepackage{textcomp}
\usepackage{url}
\usepackage[colorlinks=true,citecolor=blue]{hyperref}
\usepackage{cleveref}

\usepackage{xstring}
\usepackage{xparse}

\usepackage{tikz}
\usetikzlibrary{shadows,arrows.meta,positioning,backgrounds,fit,chains,scopes}
\usepackage{pgfplots}
\pgfplotsset{compat=1.18}

\usepackage{algorithm}
\usepackage{algorithmic}

\usepackage{enumitem}

\usepackage[maxnames=10,backref=true,style=alphabetic]{biblatex}
\newtheorem{theorem}{Theorem}
\newtheorem{lemma}[theorem]{Lemma}
\newtheorem{remark}{Remark}

\newtheorem{proposition}[theorem]{Proposition}
\newtheorem{corollary}[theorem]{Corollary}

\theoremstyle{definition}
\newtheorem{definition}[theorem]{Definition}
\newtheorem{assumption}[theorem]{Assumption}

\newcommand{\citet}[1]{\citeauthor*{#1}~\cite{#1}}
\DeclareMathOperator{\diag}{diag}

\newcommand{\kp}{\mathsf P^{s,a}}
\newcommand{\kpn}{\widehat{\mathsf P}_n^{s,a}}
\newcommand{\kpt}{\mathsf P^{s,a}_\mathrm{t}}
\newcommand{\kps}{\mathsf P^{s,a}_\mathrm{s}}
\newcommand{\kpp}{\mathsf P^{'}}
\newcommand{\kppn}{\widehat{\mathsf P}_{n}}

\newcommand{\kph}{\widehat{\mathsf{P}}^{s,a}}
\newcommand{\kphs}{\widetilde{\mathsf{P}}_n^{s,a}}
\newcommand{\Phn}{\widehat{\mathcal{P}}_{n}}
\newcommand{\Pt}{\mathcal P_{\mathrm t} \vphantom{\Phn}}
\newcommand{\BTV}[2]{\mathbb{B}_{\mathrm{TV}}(#1,#2)} 
\newcommand{\Rect}{\bigotimes_{s,a}}                               

\title{Robust Transfer Learning with Side Information}

\begin{document}

\author{Akram S. Awad\(^1\)\\ \texttt{Akram.Awad@ucf.edu}
\and Shihab Ahmed\(^1\) \\ \texttt{Shihab.Ahmed@ucf.edu}
\and Yue Wang\(^{1,2}\)\\ \texttt{Yue.Wang@ucf.edu} \and George K. Atia\(^{1,2}\) \\ \texttt{George.Atia@ucf.edu}\\ \\ \(^{1}\) Department of Electrical and Computer Engineering, University of Central Florida, USA
\\ \(^{2}\) Department of Computer Science, University of Central Florida, USA
}
\date{~}
\maketitle
\begin{abstract}
Robust Markov Decision Processes (MDPs) address environmental shift through distributionally robust optimization (DRO) by finding an optimal worst-case policy within an uncertainty set of transition kernels. However, standard DRO approaches require enlarging the uncertainty set under large shifts, which leads to overly conservative and pessimistic policies. 
 In this paper, we propose a framework for transfer under environment shift that derives a robust target-domain policy via \emph{estimate-centered} uncertainty sets, constructed through constrained estimation that integrates limited target samples with side information about the source-target dynamics. The side information includes bounds on feature moments, distributional distances, and density ratios, yielding improved kernel estimates and tighter uncertainty sets.
 The side information includes bounds on feature moments, distributional distances, and density ratios, yielding improved kernel estimates and tighter uncertainty sets.
  Error bounds and convergence results are established for both robust and non-robust value functions. Moreover, we provide a finite-sample guarantee on the learned robust policy and analyze the robust sub-optimality gap. Under mild low-dimensional structure on the transition model, the side information reduces this gap and improves sample efficiency. We assess the performance of our approach across OpenAI Gym environments and classic control problems, consistently demonstrating superior target-domain performance over state-of-the-art robust and non-robust baselines.
\end{abstract}

\section{Introduction}

Transfer reinforcement learning (RL) seeks to leverage knowledge from a source environment to accelerate and stabilize learning in a related target environment. Such a scheme is essential in real-world applications where collecting sufficient data in the target environment is costly, dangerous, or otherwise infeasible, and where policies must instead be transferred from simulation or a related operational setting.

However, the difference between the two environments,  often referred to as the \emph{sim-to-real gap} or \emph{environmental mismatch}, arise from modeling errors in simulation, unmodeled disturbances, adversarial perturbations, or non-stationary conditions, and can result in severe performance degradation when directly deploying trained policies. This underscores the need for principled transfer mechanisms that reduce adaptation time and minimize unduly data collection in the target domain.

A popular approach for transfer under model mismatch is the framework of \emph{robust MDPs} and \emph{robust RL} \cite{bagnell2001solving,nilim2004robustness,iyengar2005robust}, which constructs an uncertainty set of transition kernels centered on the source environment. By optimizing the worst-case performance, robust RL yields policies that guarantee a lower bound on return when the target kernel lies within the set, enhancing robustness when deployed to out-of-distribution environments. This enhancement makes robust RL an attractive approach to transfer. However, when the two environments are substantially different, the uncertainty set must be expanded to cover the target one, often leading to overly conservative solutions. Such pessimism can cause robust policies to underperform in the target domain.

To mitigate over-conservatism, alternative transfer methods such as multi-task learning, domain randomization, and model-free domain adaptation have been explored (see Sec. \ref{section: related work}). Multi-task RL seeks to learn representations that generalize across related tasks, while domain randomization exposes agents to diverse simulated conditions to promote robustness. Model-free adaptation methods reweight or adjust source samples to approximate the target distribution. Yet, these methods often fail when the target domain diverges sharply from the training conditions, as they do not explicitly account for the structure of uncertainty in the transition dynamics.

In this work, we develop a framework for robust transfer that explicitly leverages \emph{side information}, i.e., prior knowledge about the relationship between the source and target environments. Such information may take the form of distance constraints, density ratios, or moment conditions that capture statistical or structural similarities between domains. By integrating side information with limited target samples, we construct improved estimates of the target transition kernel, yielding policies that are closer to optimal for the target and less conservative than standard robust RL. In effect, side information reduces the sim-to-real gap by anchoring the uncertainty set around an estimated target model rather than the source, thereby decreasing the adaptation time or the data required to achieve reliable performance.

We consider both the \emph{non-robust setting}, which learns the optimal policy under the estimated target model, and the \emph{robust setting}, which guards against residual model uncertainty around this estimate. Our main contributions are:
\begin{enumerate}[label=(\arabic*), leftmargin=*, itemsep=0pt, topsep=2pt, parsep=0pt, partopsep=0pt]
    \item We develop a side-information--based framework for estimating target transition kernels and learning robust policies, integrating structural constraints into the estimation process. 
    \item We derive error bounds and establish convergence results for robust and non-robust value functions, providing asymptotic guarantees in terms of total variation distance and quality of side information. 
    \item We provide finite-sample guarantees on the robust optimality gap under mild assumptions on the transition kernel subspace, demonstrating that side information reduces suboptimality and improves sample efficiency. 
    \item Our approach is validated through experiments on OpenAI Gym and classic control tasks, showing consistent improvements over state-of-the-art baselines in both robust and non-robust settings. 
\end{enumerate}

\section{Preliminaries and Problem Setup}\label{section: pre and problem}

\noindent\textbf{Markov Decision Process (MDP).}
An MDP is denoted by the tuple $ \mathcal{M} =(\mathcal{S},\mathcal{A}, \mathsf P, r, \gamma)$, whose main components are the state space \(\mathcal{S}\), the Action space $\mathcal{A}$, the discount factor $\gamma\in (0,1)$, the reward function $r: \mathcal{S} \times \mathcal{A}\to\mathbb{R}$, the transition kernels $\mathsf P=\{\kp\in \Delta(\mathcal{S}), a\in\mathcal{A}, s\in\mathcal{S} \}$\footnote{$\Delta(\mathcal{S})$ denotes the probability simplex over $\mathcal{S}$.}, where $\kp$ is the distribution over the state space $\mathcal{S}$ of the next state conditioned on state $s$ and action $a$. Hence, $\mathsf P^{s,a} (s')$ denotes the probability of transitioning to state $s'$ if the agent is in state $s$ and takes action $a$.
At each time step $t\geq 0$, the environment transitions to a state $s_{t+1}$ according to the transition probability $\mathsf P^{s_t,a_t} (s_{t+1})$ and yields a reward $r(s_{t},a_{t})$ for the agent. \par

A stationary policy $\pi: \mathcal{S}\rightarrow\Delta(\mathcal{A}) $ is a distribution over the set of actions $\mathcal{A}$ for a given state $s$, which determines the probability of selecting a given action at a certain state. The value function of a stationary policy $\pi$ at state $s$ is defined as the expected discounted cumulative reward if the agent starts from state $s$ and takes actions according to policy $\pi$, i.e., 
\begin{align}
    V^\pi_\mathsf P(s)=\mathbb{E}_{\pi,\mathsf P}\left[\sum_{t=0}^{\infty}\gamma^t r(s_{t},a_{t})\big| S_0=s\right]. 
\end{align}
The goal of the agent is to learn a stationary policy to maximize the expected cumulative reward, i.e.,
  $ \pi^* = \arg\max_{\pi}V^{\pi}_\mathsf P.$

\noindent\textbf{Robust MDP.}
A robust MDP is defined by the tuple $(\mathcal{S},\mathcal{A}, \mathcal{P}, r, \gamma)$, where the transition kernel is not fixed but comes from an uncertainty set $\mathcal{P}$. The environment can transit to the next state according to an arbitrary transition kernel belonging to the uncertainty set.  
In this work, we consider the $(s,a)$-rectangular uncertainty set \cite{nilim2004robustness,iyengar2005robust},  
i.e., $\mathcal{P}=\bigotimes_{s,a} \mathcal{P}_{s}^{a}$, where $ \mathcal{P}_{s}^{a} \subseteq \Delta(\mathcal{S})$ and $\bigotimes_{s, a}$ is the Cartesian product over all state--action pairs. 
The robust discounted value function of policy $\pi$ at state $s$ is defined as
\begin{align}\label{eq:Vdef}
    V^\pi_\mathcal{P}(s)\triangleq \min_{\eta\in\bigotimes_{t\geq 0} \mathcal{P}} \mathbb{E}_{\pi,\eta}\left[\sum_{t=0}^{\infty}\gamma^t   r(s_t,a_t)|S_0=s\right],
\end{align}

where \(\eta=(\mathsf P_0,  \mathsf P_1, \dots)\). This accounts for the worst-case performance over the uncertainty set of transition kernels. The goal is to optimize the worst-case performance by maximizing the robust value function, i.e.,
\begin{equation}\label{eqn: robust_value}
   \pi^* = \arg\max_{\pi}V^\pi_\mathcal{P}.
\end{equation}
Equation \eqref{eqn: robust_value} finds the optimal policy that maximizes the worst-case value function over the uncertainty set of distributions $\mathcal{P}$. 

\paragraph{\textbf{Problem setup.}} We consider an environment shift scenario, where an agent is trained in a source environment and subsequently deployed in a related, but distinct, target environment. This setting is formalized using two MDPs: the source domain MDP $\mathcal{M}_\text{s} = (\mathcal{S}, \mathcal{A}, \mathsf{P}_\text{s}, r, \gamma)$ and the target domain MDP $\mathcal{M}_\text{t} = (\mathcal{S}, \mathcal{A}, \mathsf{P}_\text{t}, r, \gamma)$, which only differ in their transition dynamics, i.e., $\mathsf{P}_\text{s} \neq \mathsf{P}_\text{t}$. Both $\mathcal{S}$ and $\mathcal{A}$ are assumed to be finite sets of sizes $S$ and $A$, respectively. The agent is provided with a fixed offline dataset $\mathcal{D} = \{(s_i, a_i, s'_i)\}_{i=1}^N$, consisting of $N$ transitions. Each tuple is generated according to some distribution $\mu$ over state-action pairs, i.e., $(s_i, a_i) \sim \mu$, and the next state $s'_i$ is sampled from the nominal target transition kernel, $s'_i \sim \mathsf{P}_\text{t}^{s_i,a_i}$. In addition, we assume the availability of side information $\Phi(\mathsf{P}_\text{s}, \mathsf{P}_\text{t})$, which captures structural or statistical relationships between the source and target transition kernels, defined in Section \ref{section:  Main Approach}. 

Our objective is to learn policies for the target domain using limited offline samples from the target MDP together with side information, under (i) a non-robust setting and (ii) a robust setting that accounts for transition uncertainty.

\section{Main Approach}
\label{section:  Main Approach}
A key challenge is that the target environment is observable only through a \emph{limited offline sample}. A natural baseline is offline RL \cite{uehara2021pessimistic,wang2024achieving}, but reliable performance typically requires abundant, high-quality data; with scarce samples, offline RL often yields suboptimal policies due to coverage gaps and extrapolation error \cite{levine2020offline}. Robust MDPs address distribution shift by optimizing over an uncertainty set \emph{centered at the source} transition kernel \cite{wang2023robust,wiesemann2013robust,tamar2014scaling,lim2019kernel,xu2010distributionally}. However, when the source--target shift is large, the uncertainty set must be enlarged to include the target dynamics, inducing excessive pessimism and degraded performance in the target domain.

Our key idea is to transfer knowledge from source to target by estimating the target transition kernel using limited offline target data together with side information that encodes relationships between the domains (e.g., bounds on moments, distributional distances, or density ratios). We then optimize policies \emph{around this estimated target kernel} rather than around the source. Intuitively, if the estimate is closer to the true target dynamics than the source is, uncertainty sets centered at the estimate require smaller radii to cover the target, reducing conservatism while retaining robustness.
We study both (i) a \emph{non-robust} regime that learns the optimal policy for the estimated target model, and (ii) a \emph{robust} regime that guards against residual model uncertainty around that estimate.
\smallbreak
\noindent\textbf{Uncertainty set construction}. For radius $R\!\ge\!0$, we define the $(s,a)$-rectangular set centered at \(\mathsf P\) as
\begin{align}
\mathcal P(\mathsf P,R) \;\triangleq\; \Rect\, \BTV{\kp}{R}\:,
\end{align}

where $\BTV{p}{R}=\{q\in\Delta(\mathcal S):\|q-p\|_{\mathrm{TV}}\le R\}$ denotes the TV ball of radius $R$ centered at $p$. Hence, $\mathcal P(\mathsf P,R)$ is the Cartesian product of per-$(s,a)$ TV-balls centered at $\kp$. The non-robust regime is the special case $R=0$.

\smallbreak
\noindent\textbf{Model-based pipeline.}
Our approach is model-based and proceeds in three steps:
\begin{enumerate}[label=(\arabic*),leftmargin=*,labelsep=0.3em,itemsep=0pt,parsep=0pt,topsep=0pt,partopsep=0pt]
\item \textbf{Estimate target dynamics.} Using limited target samples and side information, compute a constrained estimator $\widehat{\mathsf P}=\{\kph\in \Delta(\mathcal{S}), a\in\mathcal{A}, s\in\mathcal{S}\}$ of the target kernel $\mathsf P_{\mathrm t}$ (Sec.~\ref{subsection: estimation with side info}).
\item \textbf{Optimize a policy.} 
\begin{align*}
\text{(Non-robust)}\quad & \pi^\star \in \arg\max_{\pi} \; V_{\widehat{\mathsf P}}^{\pi},\\[-0.25em]
\text{(Robust)}\quad & \pi^\star \in \arg\max_{\pi}\;\min_{Q\,\in\,\mathcal P(\widehat{\mathsf P},R')}\; V_{Q}^{\pi}.
\end{align*}
where $\mathcal{P}(\widehat{\mathsf P},R')$ is defined w.r.t $\widehat{\mathsf P}$.

\item \textbf{Evaluate on target.} Assess $V^{\pi^\star}$ in the target domain (non-robust) or its worst-case value over a target-domain uncertainty set (robust).
\end{enumerate}


When $\widehat{\mathsf P}$ is closer (in TV) to $\mathsf P_{\mathrm t}$ than the source $\mathsf P_{\mathrm s}$ is, the radius needed to cover the target is smaller, yielding tighter worst-case evaluations and less pessimistic policies while preserving robustness guarantees. Our theory (Sec.~\ref{sec:analysis}) quantifies this via bounds that scale with $\max_{(s,a)}\|\widehat{\mathsf P}^{s,a}-\mathsf P_{\mathrm t}^{s,a}\|_{\mathrm{TV}}$.

 Throughout this paper, we denote the target- and source-centered uncertainty sets by \(\mathcal{P}(\mathsf P_{\mathrm t},R)\) and \( \mathcal{P}(\mathsf P_{\mathrm s},R)\), respectively, and omit the dependence of the uncertainty set on the center when clear, e.g., \(\mathcal{P}_{\mathrm t}(R) = \mathcal{P}(\mathsf P_{\mathrm t},R)\),  \(\mathcal{P}_{\mathrm s}(R) = \mathcal{P}(\mathsf P_{\mathrm s},R)\), and \(\widehat{\mathcal{P}}(R) = \mathcal{P}(\widehat{\mathsf P},R)\).

Figure~\ref{fig: DRO_and_DA} overviews the setting. Panel~\ref{fig: Env_shift} shows a large environment shift between $\mathsf P_{\mathrm s}$ and $\mathsf P_{\mathrm t}$.
Panel~\ref{fig: Over_cons} illustrates the over-conservative strategy that centers an uncertainty set $\mathcal P_{\mathrm s}  (R_{\mathrm s})$ of radius $R_{\mathrm s}\!\ge\!R$ at $\mathsf P_{\mathrm s}$. 
Panel~\ref{fig: our approach} depicts our method: we center $\widehat{\mathcal{P}}(R')$ at the estimate $\widehat{\mathsf P}$, requiring a smaller $R'$ and improving target performance. 
In the non-robust case, both training and evaluation use singleton sets ($R=R'=0$).

\begin{figure*}
  \centering

  \begin{subfigure}{0.22\textwidth}
    \includegraphics[width=\linewidth]{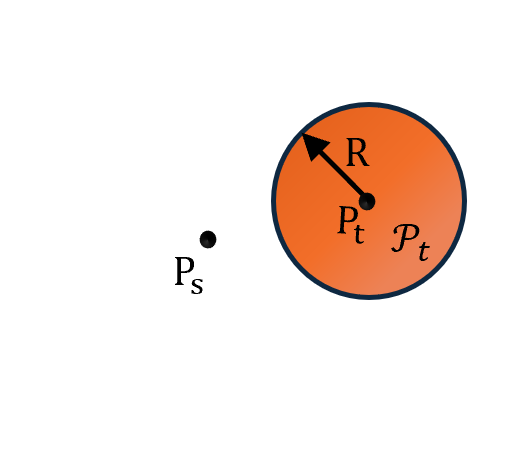}
    \caption{}
    \label{fig: Env_shift}
  \end{subfigure}
 \hspace{1cm} 
  \begin{subfigure}{0.22\textwidth}
    \includegraphics[width=\linewidth]{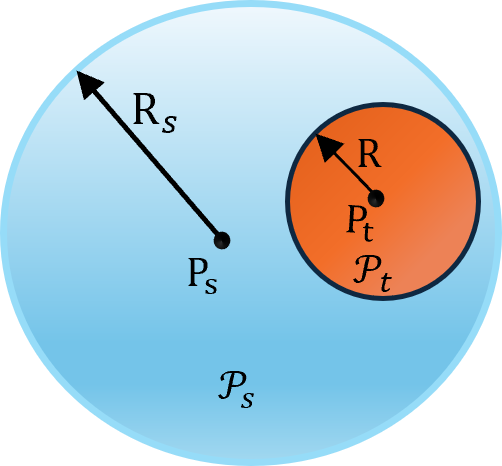}
    \caption{ }
    \label{fig: Over_cons}
  \end{subfigure}
  \hspace{1cm} 
  \begin{subfigure}{0.22\textwidth}
    \includegraphics[width=\linewidth]{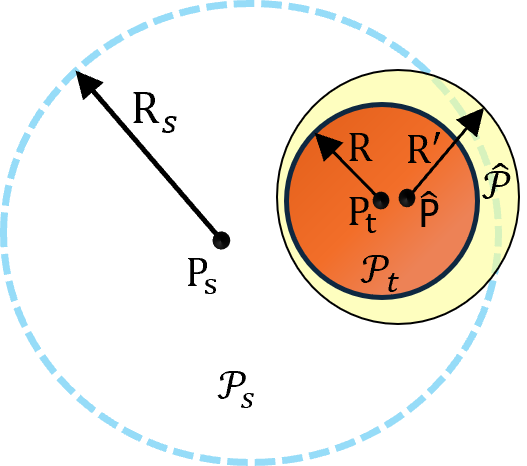}
    \caption{}
    \label{fig: our approach}
  \end{subfigure}

 \caption{\small{\textbf{(a) Environment shift:} The source and target domain environments are relatively distant. \textbf{(b) Over-conservative case:} the source uncertainty set's radius is enlarged to include the target domain, which leads to an overly conservative policy.} \textbf{(c) Our approach}: We construct the uncertainty set around the estimated target dynamics, which are closer to the true target dynamics and therefore the set requires a smaller radius.} 
  \label{fig: DRO_and_DA}
\end{figure*}

\subsection{Information-Based Estimation} 
\label{subsection: estimation with side info}

We propose an estimator for the target transition kernel that integrates limited offline target data with \emph{side information} about the relationship between source and target dynamics. Given the dataset  $\mathcal{D}=\{(s_i,a_i,s'_i)\}_{i=1}^N$
from the target MDP, let $N(s,a)=\sum_{i=1}^N \mathbf{1}\{(s_i,a_i)=(s,a)\}$ and $N_{s,a}(s')=\sum_{i=1}^N \mathbf{1}\{(s_i,a_i,s'_i)=(s,a,s')\},$ where $\mathbf{1}\{.\}$ denotes the indicator function. For each $(s,a)$, we estimate the next-state distribution $\widehat{\mathsf P}^{s,a}\in\Delta(\mathcal S)$ as
\begin{equation}
\label{eqn: s,a -est}
\widehat{\mathsf P}^{s,a}\triangleq
\begin{cases}
\begin{aligned}[t]
&\arg\max_{q\in\Delta(\mathcal S)} && \hspace{-4mm} \sum_{s'\in\mathcal S} N_{s,a}(s') \log q(s')\\
&\text{subject to}                 && \hspace{-2mm} \Phi\bigl(q,\mathsf P^{s,a}_{\mathrm s}\bigr),
\end{aligned}
& \hspace{-1mm} N(s,a)>0,\\
q_0^{s,a}, & \hspace{-2mm} N(s,a)=0.
\end{cases}
\end{equation}

where $\Phi(\cdot,\mathsf P^{s,a}_{\mathrm s})$ encodes side information tying the target to the source, and $q_0^{s,a}$ is a prior-informed default when $(s,a)$ is unseen. In our experiments we set $q_0^{s,a}=\mathsf P^{s,a}_{\mathrm s}$ or uniform $q_0^{s,a}=\mathbf{1}/S$.  
%
We refer to~\eqref{eqn: s,a -est} as the \emph{Information-Based Estimator (IBE)}.

\noindent\textbf{Forms of side information.}
We assume that, for each $(s,a)$, the source--target relationship satisfies $\Phi\!\big(\kpt,\kps\big)$. In addition to a no-side-information baseline (\textbf{Vanilla IBE}: unconstrained MLE on offline target transition counts), we instantiate $\Phi$ in four ways, yielding the variants summarized in Table~\ref{tab: estimation methods}. The constants ($d_{s,a},\beta_{s,a},B_{s,a}$) below may be per-$(s,a)$ or global. 

\begin{enumerate}[label=(\arabic*),leftmargin=0pt,itemindent=1.5em,
itemsep=0pt,parsep=0pt,topsep=0pt,partopsep=0pt]
\item \textbf{Distance IBE.}
Bound the discrepancy to the source \(\mathrm{dist}\!\big(q,\kps\big)\le d_{s,a}\),
where $\mathrm{dist}$ is either total variation $\|q-\kps\|_{\mathrm{TV}}$ or Wasserstein-1 $W_1(q,\kps)$ with a specified ground cost (Appendix~\ref{Appendix: def.}), reflecting the assumption $\mathrm{dist}\!\big(\kpt,\kps\big)\le d_{s,a}$. In robotics and control, such bounds arise naturally from known limits on physical parameter variations (e.g., friction coefficients, actuator gains, or payload mass estimated via calibration) and can be derived formally from Lipschitz continuity of the dynamics in the parameters (Appendix~\ref{Appendix: side_info}). In finite state spaces, the radius
 $d_{s,a}$ can also be estimated from source and target samples via minimax-optimal divergence estimators \cite{jiao2018minimax,nguyen2010estimating}.

\item \textbf{Moment IBE.}
Constrain feature moments. Let $\phi:\mathcal S\to\mathcal H$, mapping the state space to a Hilbert space, and $\mu(\mathsf P)=\mathbb E_{s'\sim \mathsf P}[\phi(s')]$. Impose \(|\mu(q)-\mu(\kps)|\le \beta_{s,a}\),
matching the assumption $|\mu(\kpt)-\mu(\kps)|\le \beta_{s,a}$. 
This captures coarse aggregate information when full transition distributions are unavailable--- for instance, known bounds on average velocity or energy dissipation in control systems---and is standard in empirical likelihood and conditional moment restriction frameworks \cite{kremer2022functional}.

\item \textbf{Density IBE.}
Assume absolute continuity and a bounded density ratio on the relevant support: \(0\le \frac{\kpt(s')}{\kps(s')}\le B_{s,a}\quad (\text{whenever }\kps(s')>0)\),
and enforce the element-wise constraint \(0\le q(s')\le B_{s,a}\,\kps(s'),~ \forall s'\in\mathcal S\). The cap $B_{s,a}$ encodes practical limits on importance reweighting under distribution shift, preventing excessive variance from extreme weights, and can be estimated from paired source--target samples via ratio estimation methods~\cite{kanamori2009least,nguyen2010estimating}. If support mismatch is possible, we pre-smooth $\kps$ with small pseudocounts.

\item \textbf{Low-Dimensional Structure (LDS) IBE.}
For each $(s,a)$, let $\{\kp_\theta\}_{\theta\in\Theta\subset\mathbb R^{\mathsf d}}$ be a parametric family (e.g., softmax). Write
$\kps=\kp_{\theta_{\mathrm s}}$ and $\kpt=\kp_{\theta_{\mathrm t}}$.
Let $\mathcal I_{\mathrm{shared}}\subseteq\{1,\dots,\mathsf d\}$ be the coordinates shared by source and target, and define the affine subspace \(\Theta_0 \;\triangleq\; \bigl\{\theta\in\Theta:\ \theta(\mathcal I_{\mathrm{shared}})=\theta_{\mathrm s}(\mathcal I_{\mathrm{shared}})\bigr\},\)
whose intrinsic dimension is $\mathsf d_0 = \mathsf d - |\mathcal I_{\mathrm{shared}}| \ll \mathsf d$. Such shared coordinates arise naturally when source and target differ in only a subset of physical parameters (for instance, when kinematics are preserved across domains but actuator gains or payload vary) so that the intrinsic degrees of freedom governing the shift are low-dimensional. Estimate the remaining $\mathsf d_0$ free coordinates by constrained MLE (CMLE) on offline target counts: 
\[
\widehat\theta_{\mathrm t}\in
\arg\max_{\theta\in\Theta_0}\ \sum_{s'\in\mathcal S} N_{s,a}(s')\log \kp_\theta(s'),
~~
\widehat{\mathsf{P}}^{s,a}=\kp_{\widehat\theta_{\mathrm t}}\!.
\]
If $N(s,a)=0$, use the uniform-initialized default \(1/S\).
\end{enumerate}

\textsc{LDS-IBE} explicitly adopts the softmax parameterization and enforces a low-dimensional structure on the transition kernels. As discussed later in Sec.~\ref{section: opt_gap}, this structural prior facilitates finite-sample guarantees and a suboptimality-gap analysis.

We note that stronger priors (tighter $\Phi$) shrink the feasible set in~\eqref{eqn: s,a -est}, yielding more concentrated estimators and improved lower bounds on variance, which we formalize using constrained Cramér–Rao bounds in Appendix \ref{Appendix: Cramer-Rao}. 

\begin{table*} 
\centering
\caption{\small{Information-Based Estimation (IBE) under different side information $\Phi$. } }


 \scalebox{0.70}{
\begin{tabular}{||c|c|c|c||}\hline\hline
    Distance IBE & Density IBE & Moment IBE & LDS IBE\\ \hline\hline
      \makecell{\(\max_{q \in \Delta(\mathcal{S})} \sum_{s'\in \mathcal{S}} N_{s,a}(s') \log q(s') \) \\ \\ s.t. \( \quad \text{dist}(q,\mathsf P^{s,a}_\text{s}) \leq d_{s,a}\)}  & 
       
       \makecell{\(\max_{q \in \Delta(\mathcal{S})} \sum_{s'\in \mathcal{S}} N_{s,a}(s') \log q(s') \) \\ \\ s.t. \(q/ \mathsf P^{s,a}_\text{s} \leq B_{s,a}\)}      & 
       
        \makecell{\(\max_{q \in \Delta(\mathcal{S})} \sum_{s'\in \mathcal{S}} N_{s,a}(s') \log q(s') \) \\ \\ s.t. \(|\mu(q)- \mu(\kps)|\leq \beta_{s,a}\) }      &

        \makecell{\(\max_{\theta \in \mathbb{R}^{\mathsf{d_0}}} \sum_{s'\in \mathcal{S}} N_{s,a}(s') \log \kp_{\theta}(s') \) }       \\ \hline

\end{tabular}}\label{tab: estimation methods}

 \end{table*} 

\begin{remark}
\label{remark:value-aware-ibe}
In Appendix~\ref{Appendix: Value-Aware}, we also consider a \emph{value-aware} side-information constraint based on a Wasserstein radius with the task-relevant pseudometric
$d_V(s,s') \triangleq |V^{\pi}_{\mathcal P}(s)-V^{\pi}_{\mathcal P}(s')|$.
Intuitively, it measures shift in the geometry that matters for control: moving probability
between states of similar value is cheap, while moves across large value gaps are expensive. 
This is relevant in sim-to-real settings where small perturbations in robot position or sensor readings leave task value nearly unchanged.
\end{remark}
\noindent\textbf{Policy estimation and evaluation.}
After obtaining $\widehat{\mathsf P}$, we compute policies by Value Iteration (VI) (see Algorithm~\ref{alg:valueiteration}) in both regimes using standard (robust) Bellman updates. 
In the non-robust setting, we apply VI with $\widehat{\mathsf P}$.
In the robust setting, we use $(s,a)$-rectangular TV balls centered at the estimate,
 \(\widehat{\mathcal{P}}(R)\),
and perform robust VI with support-function evaluations over \(\widehat{\mathcal{P}}(R)\). 
For evaluation, policies are tested in the target domain via Algorithm~\ref{alg:evaluatoin}: we report target-domain value (non-robust) and worst-case value over target-centered uncertainty sets \(\mathcal{P}_{\mathrm t}(R)\) (robust); in the special case $R=0$ both reduce to the non-robust evaluation.

\section{Theoretical Analysis}
 \label{sec:analysis}
We provide finite-sample and asymptotic guarantees for the side-information--guided estimator and the induced (robust and non-robust) policies. We consider a discounted MDP with $\gamma\in(0,1)$ and rewards in $[0,1]$. For exposition, we assume balanced coverage: for each $(s,a)$, we observe $n$ i.i.d. target transitions
$x^{s,a}_1,\ldots,x^{s,a}_n \sim \kpt$ and denote by $\kpn$ its IBE estimate. Collecting these estimates over all (s,a) pairs yields $\kppn=\{\kpn\in\Delta(\mathcal S): s\in\mathcal S, a\in\mathcal A\}$. Define the estimate-centered uncertainty set \(\Phn(R)\triangleq \bigotimes_{(s,a)} \mathbb{B}_{\mathrm{TV}}(\kpn,R),\)
making the dependence on $n$ explicit. The implementation does not require balanced coverage and instead uses empirical counts $N(s,a)$ from the offline target data.

For any policy $\pi$ and uncertainty set $\mathcal P(R)$, let
$V_{\mathcal P}^{\pi}=\inf_{Q\in\mathcal P(R)} V_Q^\pi$
denote the robust value. Define the target-centered and estimate-centered robust-optimal policies by \(\pi^\star\in\arg\max_\pi V_{\Pt}^{\pi},
\pi_n\in\arg\max_\pi V_{\Phn}^{\pi}\), respectively. 
We also write $V_{\Phn}^{\pi_n}$ and $V_{\Pt}^{\pi_n}$ for the worst-case values of $\pi_n$ evaluated on $\Phn(R)$ and $\Pt(R)$, respectively. The non-robust setting corresponds to $R=0$, where $\Pt(0)$ and $\Phn(0)$ reduce to singletons. Proofs of all results are deferred to Appendix~\ref{Appendix: Proofs}.

\subsection{IBE for Transfer}
\label{section:value-convg}
We show that IBE-based policies are consistent for the target domain. 
Let $\delta_n \triangleq \max_{(s,a)} \lVert \kpn - \kpt \rVert_{\mathrm{TV}}$.

\noindent\textbf{Training and evaluation errors.} We distinguish (i) the \emph{training error}, which compares the estimated-policy value on the estimate-centered set to the target-optimal value; and (ii) the \emph{evaluation error}, which evaluates the estimated policy on the \emph{target-centered} set.

\begin{theorem}[Training error]
\label{thm: train error bound}
For rewards in $[0,1]$ and any $\gamma\in(0,1)$,
\begin{align}
\bigl\| 
V^{\pi_n}_{\Phn}
-
V^{\pi^\star}_{\mathcal{P}_t\vphantom{\Phn}}
\bigr\|_\infty \le\; \frac{2\,\delta_n}{(1-\gamma)^2}
\end{align}
\end{theorem}
The training error scales linearly with the uniform TV error $\delta_n$. Hence, the closer $\kpn$ is to $\kpt$ uniformly over $(s,a)$, the closer the training value of $\pi_n$ is to the target-optimal robust value. As the estimator improves (e.g., via side information), $\delta_n\!\downarrow$ and $V^{\pi_n}_{\Phn} \to V^{\pi^\star}_{\Pt}$.

\begin{theorem}[Evaluation error]
\label{thm: eval error bound}
Under the same conditions,
\begin{align}\label{eqn: eval_err}
\bigl\| V^{\pi_n}_{\Pt} - V^{\pi^\star}_{\Pt} \bigr\|_{\infty}
\;\le\; \frac{4\,\delta_n}{(1-\gamma)^2}\,.
\end{align}
\end{theorem}
This bounds the \emph{deployment-time} gap when evaluating $\pi_n$ on the \emph{target-centered} set. It is at most twice the training bound: moving from the estimate-centered set (used to learn $\pi_n$) to the target-centered set costs a factor of at most $2$. Consequently, as $\delta_n\to 0$ the evaluation error vanishes, guaranteeing asymptotic optimality on the target-centered robust problem. The dependence of \(\delta_n\) on the form of side information \(\Phi\) is discussed in Appendix~\ref{Appendix: Dependence on sideinfo}.

\begin{corollary}[Consistency]
\label{cor:consistency}
If $\delta_n \to 0$ (e.g., $\kpn \to \kpt$ in TV), then
\(
\| V^{\pi_n}_{\Pt} - V^{\pi^\star}_{\Pt} \|_\infty \to 0,
\)
so the IBE-learned policy is asymptotically optimal for the target-centered robust problem. 
\end{corollary}

Thus, convergence of the robust (and non-robust) value functions follows if the transition-kernel estimates converge in total variation. The next result shows that this holds for IBE.

\begin{proposition}[TV-consistency of IBE]
\label{thm: CMLE in TV}
For fixed $(s,a)$, let $\kpn$ be the IBE solution of \eqref{eqn: s,a -est} with a side-information constraint set $\Phi(\cdot,\kps)$ from Table~\ref{tab: estimation methods}. Assume $x^{s,a}_1,\dots,x^{s,a}_n \overset{\text{i.i.d.}}{\sim} \kpt$ and that the true target kernel lies in the \emph{relative interior} of the feasible set. Then,
$
\bigl\|\kpn - \kpt\bigr\|_{\mathrm{TV}} \overset{n\to\infty}{\longrightarrow} 0.
$
If this holds for every $(s,a)$ with per-pair sample sizes $n\to\infty$, then $\delta_n \to 0$.
\end{proposition}

Proposition~\ref{thm: CMLE in TV} shows that IBE is \emph{TV-consistent}; the estimates converge to the true target kernel in total variation. TV-consistency is a sufficient condition for value-function consistency via Theorems~\ref{thm: train error bound}--\ref{thm: eval error bound}: both the training and evaluation gaps scale as $O(\delta_n)$, with $\delta_n\to\!0$. Thus, the IBE-learned policy becomes asymptotically optimal for the target-centered robust problem, and the non-robust case follows as the special instance $R=0$. Proposition~\ref{thm: CMLE in TV} extends classical MLE consistency to the \emph{constrained} setting induced by side information (Table~\ref{tab: estimation methods}): TV/W$_1$ balls, moment sets, density-ratio boxes, and LDS subspaces all yield closed, convex feasible regions on which the constrained MLE remains consistent. This validates our transfer pipeline of estimating target dynamics with side information and then optimizing around that estimate, and explains the empirical improvements we observe in the target domain. We corroborate TV-consistency and the predicted error decay in Appendix~\ref{Appendix: convergence_exp}, and we empirically verify the bound from Theorem~\ref{thm: eval error bound} in Appendix~\ref{Appendix: eval_error_exp}.

\subsection{Suboptimality Gap and Finite-Sample Guarantees}
\label{section: opt_gap}

We now leverage IBE's consistency to obtain \emph{finite-sample} guarantees for the learned \emph{robust} policy, focusing on the LDS--IBE case. Under a mild structural assumption on the transition family, we choose a radius $R_n$ so that the \emph{estimate-centered} uncertainty set $\Phn(R_n)$ (as defined earlier) contains the true target kernel with high probability. This yields a high-confidence \emph{out-of-sample} lower bound on target performance and, consequently, a bound on the robust suboptimality gap.

\begin{assumption}[Parametric family and TV--Lipschitzness]
\label{assumption: Lipschitzness}
There exists $\Theta\subset\mathbb R^{\mathsf d}$ and a mapping $\theta\mapsto \mathsf P_\theta$ such that for some $L>0$, \(\)
\[
\bigl\|\kp_\theta-\kp_{\theta'}\bigr\|_{\mathrm{TV}}
\ \le\ L\,\|\theta-\theta'\|_{1}
\quad \forall\,\theta,\theta'\in\Theta,\ \forall(s,a),
\]
and the target kernel satisfies $\kpt=\mathsf P_{\theta_{\mathrm t}}$ for some $\theta_{\mathrm t}$ in the LDS subspace $\Theta_0$ of intrinsic dimension $\mathsf d_0\ll \mathsf d$.
\end{assumption}

\begin{lemma}[Finite-sample radius for LDS--IBE]
\label{lemma:radius}
Under Assumption~\ref{assumption: Lipschitzness} (TV--Lipschitz in $\ell_1$), for any $\delta_1\in(0,1)$ there exists $C_0>0$ such that choosing 
\begin{align}
   R_n \;=\; L C_0\sqrt{\frac{{\mathsf d}_0}{n}}
\end{align}

ensures $\mathsf{P}_{\mathrm t}\in \Phn(R_n)$ with probability at least $1-\delta_1$.
\end{lemma}

By construction, this gives a high-probability \emph{out-of-sample} lower bound: for any policy $\pi$,
$
V^{\pi}_{\mathsf{P}_{\mathrm{t}}}\ \ge\ V^{\pi}_{\Phn(R_n)}
$
with probability $\ge 1-\delta_1$.
LDS shrinks the effective dimension from $\mathsf d$ to $\mathsf d_0$, tightening $R_n$.

\noindent\textbf{Robust suboptimality gap.}
Let $\pi_n\in\arg\max_{\pi} V^{\pi}_{\Phn}$ and $\pi^\star\in\arg\max_{\pi} V^{\pi}_{\Pt}$. Define the (nonnegative) gap
$
\mathrm{Gap}(\pi)\triangleq V^{\pi^\star}_{\Pt}-V^{\pi}_{\Pt}
$, where both \(\Pt\) and \(\Phn\) are of radius \(R_n\)

\begin{theorem}[Suboptimality gap]
\label{thm: optgap}
Under Assumption~\ref{assumption: Lipschitzness} and the radius choice above, for rewards in $[0,1]$ and any $\gamma\in(0,1)$,
\begin{align}
  \mathrm{Gap}(\pi_n)
\ \le\ \frac{R_n}{(1-\gamma)^2}\,
\ =\
\tilde O\!\left(\frac{L }{(1-\gamma)^2} \sqrt{\frac{\mathsf d_0}{n}}\right)  
\end{align}

with probability at least $1-\delta_1$, where $\tilde O(.)$ is up to logarithmic factors.
\end{theorem}

Hence, combining the high-probability coverage $\mathsf{P}_{\mathrm{t}}\in\Phn(R_n)$ with the evaluation bound of Theorem~\ref{thm: eval error bound} yields a target-domain suboptimality gap that decays as $O(\sqrt{\mathsf d_0/n})$ (up to logs), rather than the $O(\sqrt{\mathsf d/n})$ rate obtained without side information. Theorem~\ref{thm: optgap} makes this improvement explicit: by incorporating low-dimensional structure through LDS-IBE, the uncertainty radius $R_n$ required to cover the target dynamics is smaller, which directly translates into a tighter suboptimality gap.

\section{Related Work}\label{section: related work}
\noindent\textbf{Robust RL.}
Robust RL casts robustness as distributionally robust optimization (DRO), maximizing worst-case return over an uncertainty set of transition kernels \cite{iyengar2005robust,nilim2004robustness,bagnell2001solving,wiesemann2013robust,lim2013reinforcement,tamar2014scaling}. Methods are either \emph{model-based}, i.e., fit a model then apply robust dynamic programming \cite{lim2019kernel,yu2015distributionally,yang2021towards,shi2023curious,wang2023robust}, or \emph{model-free}, that is, optimize a robust objective without explicit modeling \cite{roy2017reinforcement,si2020distributionally,zhou2021finite,wang2021online,ho2021partial,badrinath2021robust,wang2023sample,liu2022distributionally,wang2024model,wang2023model}. Both typically rely on data from a single source environment. While worst-case guarantees yield lower bounds when the target kernel lies in the set, large source--target shifts force wide uncertainty sets, producing conservative policies that underperform in the target domain.

\noindent\textbf{Multi-Task RL.}
Multi-task RL learns shared representations across tasks via joint training or transferable features to improve generalization across environments \cite{cheng2022provable,agarwal2023provable,huang2022provably,du2021bilinear}. Several works use this for \emph{transfer}, aiming to boost a designated target task: \citet{cheng2022provable} performs reward-free exploration on source tasks to learn representations, then adapts \emph{online} on the target; \citet{agarwal2023provable} strengthens this with a general state-dependent linear model and a cross-sampling scheme for in-distribution generalization. Related frameworks provide provable multi-task representation learning guarantees \cite{huang2022provably,du2021bilinear}. In contrast, our setting assumes only \emph{limited offline target data} (no online interaction at deployment) and leverages side information to estimate target dynamics before (robustly) optimizing a policy.

\noindent\textbf{Domain Randomization and Meta Learning.}
Domain randomization improves robustness by training over a wide distribution of simulated environments prior to deployment \cite{peng2018sim,tobin2017domain,chebotar2019closing}. Meta-RL trains across related tasks to learn fast adaptation from a small amount of \emph{online} target experience \cite{finn2017model,nagabandi2018learning}. These approaches typically rely on broad simulator coverage or interaction at deployment; performance can degrade when the target lies outside the randomized/meta-training distribution or when online adaptation is limited. In contrast, we assume only \emph{limited offline target data} and no online interaction, and we use side information to estimate target dynamics and (robustly) optimize the policy.

\noindent\textbf{Domain Adaptation (DA).}
Prior DA methods mitigate source--target mismatch by transferring samples or reusing source data, often assuming shared dynamics 
\cite{taylor2008transferring,laroche2017transfer}.
We instead consider \emph{transition mismatch} and leverage side information to estimate target kernels, mitigating negative transfer under limited target data.
\citet{tirinzoni2018importance} extend FQI via importance weighting across multiple sources, fitting weights with Gaussian Processes 
, which limits applicability in discrete or non-Gaussian settings;
our approach is agnostic to the transition model class.
\citet{wen2024contrastive} rank and filter source transitions using a mutual-information--based domain gap 
but do not address robustness to uncertain target dynamics; we explicitly model uncertainty sets around the estimated target kernel.

\section{Numerical Experiments}\label{section: numerical experiments}
We evaluate our approach against state-of-the-art baselines on six benchmarks. Due to space, we report \textit{CartPole} here and defer the remaining tasks to Appendix~\ref{Appendix: Additional exp}, which display similar trends.

\noindent\textbf{Baselines.} We compare to FQI \cite{ernst2005tree}, Importance-Weighted FQI (IWFQI) \cite{tirinzoni2018importance}, IGDF \cite{wen2024contrastive}, and standard Q-learning. FQI and Q-learning use only \emph{target} samples; IWFQI and IGDF use both source and target data. The original IWFQI estimates weights via Gaussian Processes (GPs), restricting it to continuous states. To isolate its best-case performance, we instead feed it oracle (true) importance weights. Our method first estimates the target kernel with IBE, then computes an optimal policy via non-robust or robust RL (Appendix~\ref{Appendix:algorithms}). We report performance in the target domain under both evaluations.

\noindent\textbf{Side information and environments.}
We instantiate IBE using the constraints in Table~\ref{tab: estimation methods}: Distance IBE (TV and $W_1$ with bounds set to the true source--target distances, i.e., $d_{s,a}=\|\mathsf P^{s,a}_{\mathrm s}-\mathsf P^{s,a}_{\mathrm t}\|_{\mathrm{TV}}$ and $W_1(\mathsf P^{s,a}_{\mathrm s},\mathsf P^{s,a}_{\mathrm t})=d_{s,a}$), Moment IBE (moment gap $\beta_{s,a}= |\mu({\mathsf P^{s,a}_{\mathrm t}})-\mu({\mathsf P^{s,a}_{\mathrm s})}|$), Density IBE ``global" ($B_{s,a}=\max(\mathsf P^{s,a}_{\mathrm t}/\mathsf P^{s,a}_{\mathrm s})$), and Density IBE ``local'' (elementwise $B_{s,a}=\mathsf P^{s,a}_{\mathrm t}/\mathsf P^{s,a}_{\mathrm s}+1$). We test on three toy-text tasks (Frozen Lake, Cliff Walking, Taxi) and three classic control problems (Acrobot, Cart Pole, Pendulum) from OpenAI Gym \cite{brockman2016openai}. In all cases, source and target differ only in their transition kernels. Environment details are in Appendix~\ref{Appendix: Env description}.

\noindent\textbf{Sampling and evaluation.}
We draw $N$ state--action pairs uniformly and record counts $N(s,a)$. For each occurrence, the next state is sampled from the target kernel. We vary per-pair sample sizes over $\{1,5,10,50,100,150,\ldots,10^4\}$. Kernel estimates are initialized at the source kernel, except for Vanilla IBE, which uses a uniform initializer to assess learning purely from target data. Results are averaged over 20 runs with 95\% confidence intervals.

\noindent\textbf{Performance results.}
We evaluate IBE in both non-robust and robust regimes. Transition kernels are estimated via \eqref{eqn: s,a -est}. Policies are computed with Algorithm~\ref{alg:valueiteration} (Appendix~\ref{Appendix:algorithms}). We set the robustness radius to $R=0$ (non-robust) and $R=0.1$ (robust), and evaluate in the target domain with the same radius using Algorithm~\ref{alg:evaluatoin}. Due to space, we show \textit{CartPole} in the main text and defer the remaining environments to Appendix~\ref{Appendix: non_robust_exp}. Performance is reported as the state-averaged value $\frac{1}{|\mathcal S|}\sum_{s\in\mathcal S} V^{\pi}(s)$ across varying target-sample budgets, averaged over 20 runs.
Figures~\ref{fig: Non-robust scenario} and \ref{fig: robust scenario} plot non-robust and robust results, respectively: left panels compare IBE side-information variants; right panels compare our best IBE to SOTA baselines. Note that IGDF's assumptions fail for very small $N$ (e.g., $N\in\{1,5\}$), so head-to-head comparisons with IGDF begin at $N\ge 150$, where all methods are well-defined.

As shown in the left panels of Figures~\ref{fig: Non-robust scenario} and \ref{fig: robust scenario}, incorporating side information markedly improves target-domain performance over \emph{Vanilla IBE} (samples only). To isolate the effect beyond initialization, we also initialize Vanilla IBE at the source kernel. The results indicate that the side information (not just initialization) drives the gains.  Among our variants, \emph{Density IBE (local)} performs best, plausibly because element-wise density-ratio bounds provide sharper guidance on the target transitions. In the right panels, \emph{Density IBE (local)} also outperforms SOTA baselines in both non-robust and robust regimes, and the same is observed for other IBE variants, such as moment-IBE. We also compare to an over-conservative baseline that centers the uncertainty set at the \emph{source} kernel and enlarges the radius to cover the target, $R+\|\mathsf P^{s,a}_{\mathrm s}-\mathsf P^{s,a}_{\mathrm t}\|_{\mathrm{TV}}$ (per $(s,a)$). As expected, this pessimistic construction yields poor target performance (see Figure~\ref{fig: robust scenario}).

 \begin{figure}[t]
    \centering
    \includegraphics[width=.45\linewidth]{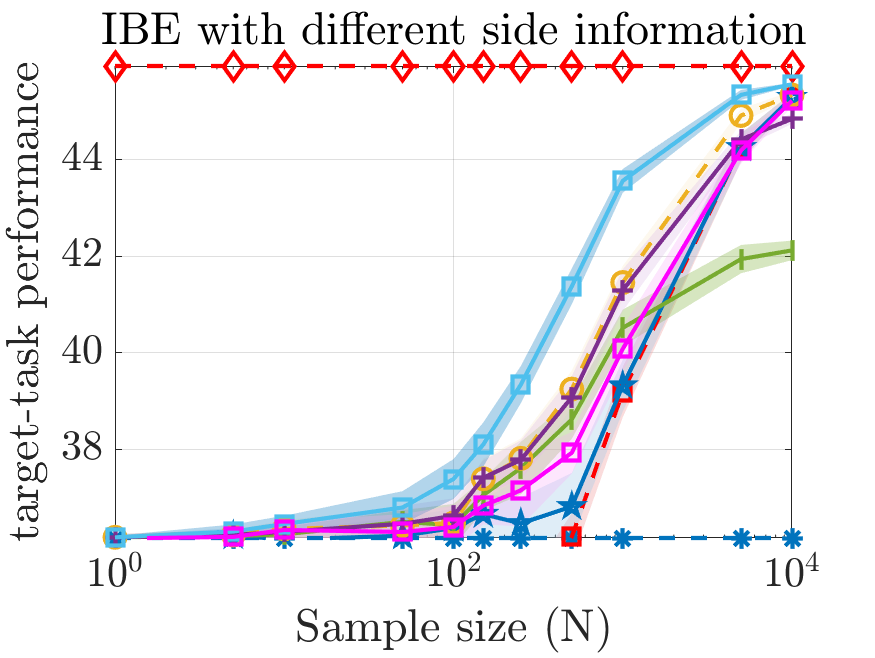}%
     \hspace{-0.1cm}
    \includegraphics[width=.45\linewidth]{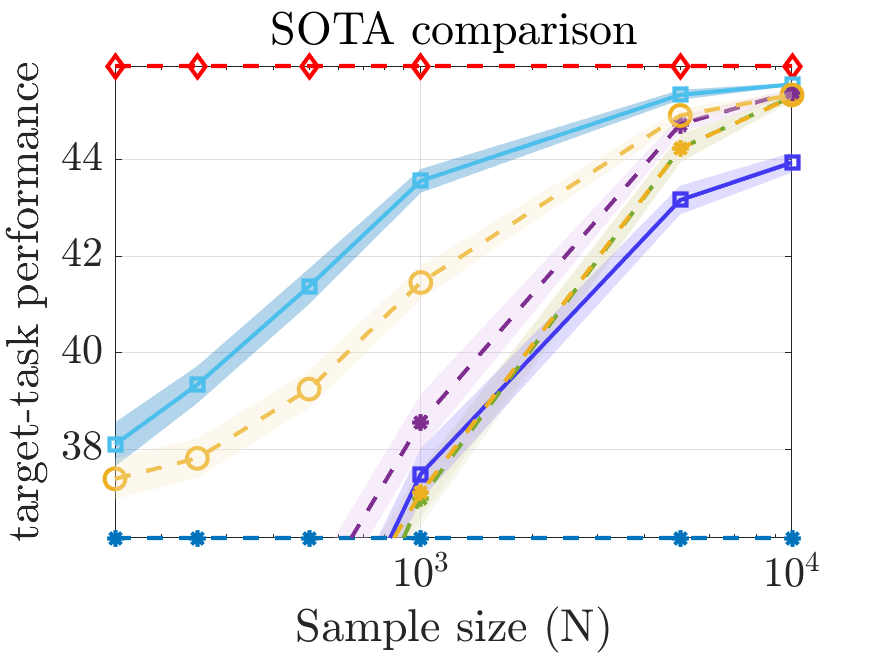}
   \includegraphics[width=.9\textwidth]{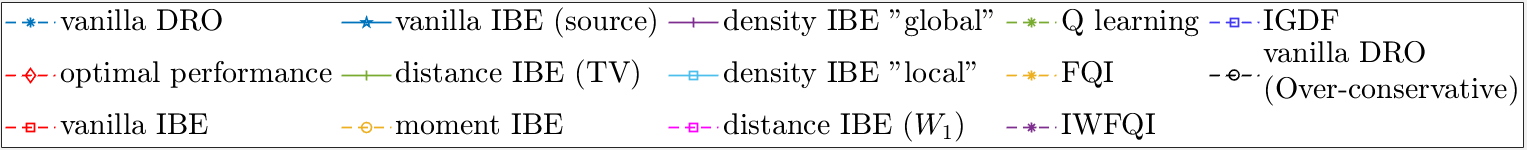}
    \caption{\small{Target domain performance for the non-robust setting as a function of sample size for \textit{CartPole}. The legend is shared between Figures ~\ref{fig: Non-robust scenario} and \ref{fig: robust scenario}.}}
    \label{fig: Non-robust scenario}
  \end{figure}

\begin{figure}[t]
    \centering
    \includegraphics[width=.45\linewidth]{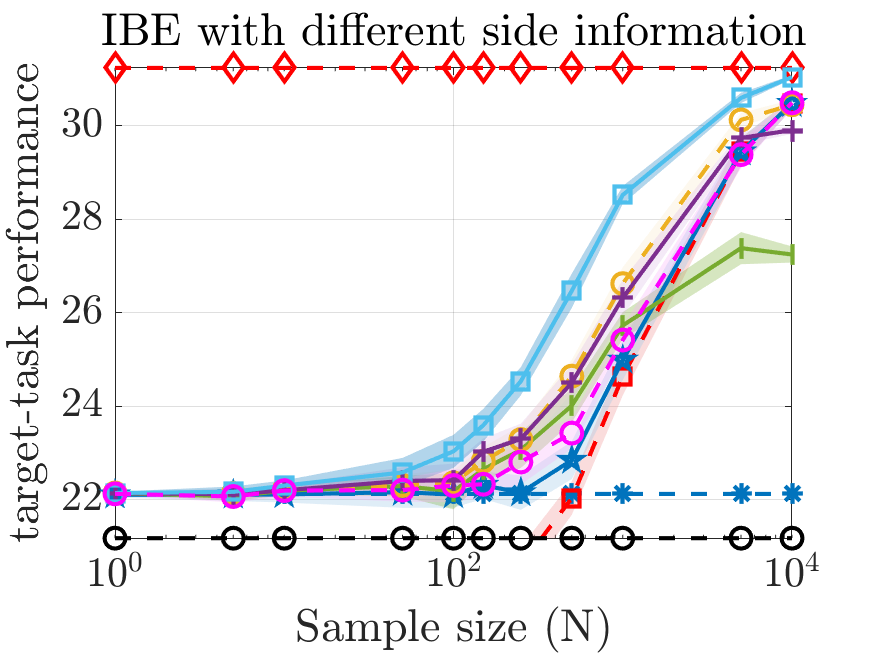}%
     \hspace{-0.1cm}
    \includegraphics[width=.45\linewidth]{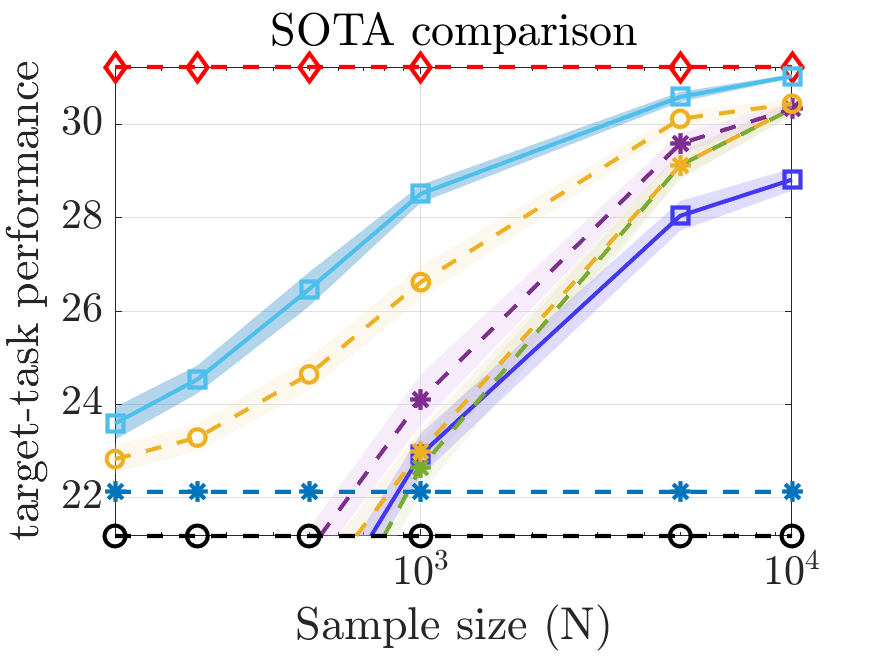}

   \caption{\small{Target domain performance for the robust setting as a function of sample size for \textit{CartPole}.}}
    \label{fig: robust scenario}
  \end{figure}

\noindent\textbf{Dimension effect.}
We empirically test the dimension effect of LDS side information predicted by our theory. 
In \textit{CartPole}, we parameterize the dynamics with a softmax model and follow the LDS-IBE procedure in Sec.~\ref{subsection: estimation with side info} to estimate $\theta_{\mathrm t}$. 
We compare LDS-IBE (which constrains $\theta$ to the known low-dimensional subspace) to Vanilla IBE (no side information). 
Although the ambient parameter dimension is $\mathsf d=4$ (Appendix~\ref{Appendix: Env description}), the logits depend on $\theta$ through a $\mathsf d_0=2$ subspace, so the effective model dimension is two.
Using the estimated dynamics, we learn a policy and evaluate it on the target environment in both non-robust and robust regimes. Figure~\ref{fig: subopt} reports the suboptimality gap versus the number of target samples $N$ (left: non-robust, right: robust). 
As expected, LDS--IBE yields consistently smaller gaps than Vanilla IBE, reflecting the improved \(\tilde O\!\big(\sqrt{\mathsf d_0/N}\big) \) dependence from exploiting the low-dimensional structure in the transition dynamics.
\begin{figure}[t]
  \centering

  \begin{minipage}[t]{0.45\linewidth}
    \centering
    \includegraphics[width=\linewidth]{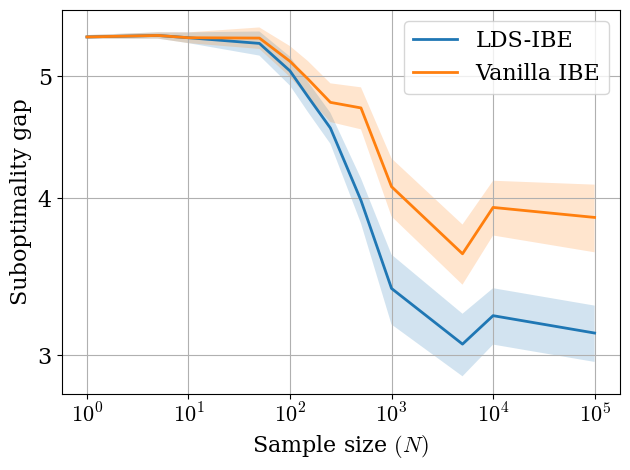}
   
  \end{minipage}
  \begin{minipage}[t]{0.45\linewidth}
    \centering
    \includegraphics[width=\linewidth]{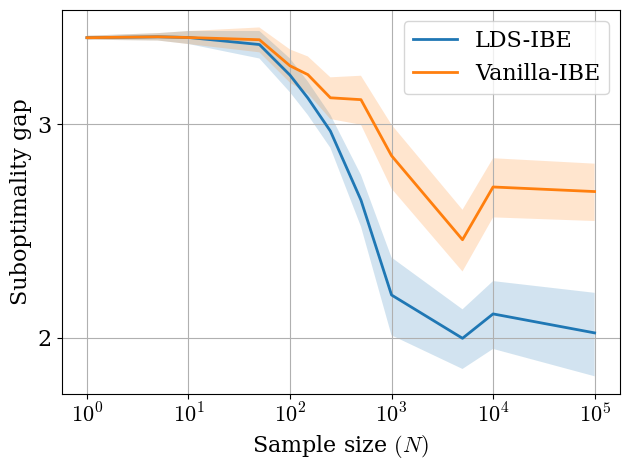} 
  \end{minipage}
  \caption{\small{Suboptimality gap as a function of sample size \((N)\) (log-log scale) in the CartPole environment for LDS-IBE, for non-robust (left) and robust (right) scenarios.}
}
 \label{fig: subopt}
\end{figure}

\section{Conclusion}
\label{sec:conc}
We proposed a model-based transfer RL framework leveraging limited offline target data and side information to estimate target dynamics and learn effective policies. Centering uncertainty sets at an information-based target estimate (IBE) rather than the source avoids the over-conservatism of source-centered DRO, yielding tighter radii and less pessimistic policies. Incorporating side-information constraints 
brings the estimate closer to the true target kernel, reducing the data needed for adaptation.
Theoretically, we established value-function error bounds, scaling with the estimator's uniform TV error, and finite-sample guarantees. Under LDS, the robust suboptimality gap scales as $\tilde O(\sqrt{\mathsf d_0/n})$, explicitly quantifying the sample-efficiency gains from side information. Empirically, IBE---particularly its density-ratio and moment variants---consistently outperforms RL baselines in both robust and non-robust scenarios.

\newpage

\printbibliography
\clearpage

\appendix
\onecolumn

\section{Definitions}\label{Appendix: def.}

\begin{definition}[Total variation distance \cite{LevinPeresWilmer2017}]
Let $(\mathcal{X},\mathcal{F})$ be a measurable space and let $\mathsf{P},\mathsf{Q}$ be two probability distributions defined on it. The total variation (TV) distance is defined as
\[
\|\mathsf{P}-\mathsf{Q}\|_{TV}\;:=\;\sup_{A\in\mathcal{F}} \bigl|\mathsf{P}(A)-\mathsf{Q}(A)\bigr|.
\]

In the discrete case, $\|\mathsf{P}-\mathsf{Q}\|_{TV}=\tfrac12\sum_{x\in\mathcal{X}} |\mathsf{P}(x)-\mathsf{Q}(x)|$.
\end{definition}

\begin{definition}[Wasserstein--1 (Earth Mover's) distance \cite{villani2008optimal}]
Let $(\mathcal{X},\|\cdot\|)$ be a metric space and let $\mathsf{P},\mathsf{Q}$ be probability measures on $\mathcal{X}$ with finite first moments.
The Wasserstein--1 distance is
\[
W_1(\mathsf{P},\mathsf{Q})\;:=\;\inf_{\rho\in\mathsf{\Pi}(\mathsf{P},\mathsf{Q})} \int_{\mathcal{X}\times\mathcal{X}} \|x-y\|\,d\rho(x,y),
\]
where $\Pi(\mathsf{P},\mathsf{Q})$ is the set of couplings (joint distributions) with marginals $\mathsf{P}$ and $\mathsf{Q}$. For the cost $\|.\|$, we use the Euclidean distance, i.e., \(\|x-y\|_{2}\), in this paper.

\end{definition}

\begin{definition}[Value-Aware 1-Wasserstein distance ] \label{def: VAW}
Let \(d_v(s,s') = |V^{\pi}_{\mathcal P}(s)-V^{\pi}_{\mathcal P}(s')|\) be a pseudometric, and \((\mathcal{S}, d_v)\) be a finite pseudometric space. The \emph{Value-Aware 1-Wasserstein distance} between \(\mathsf{P}\) and \(\mathsf{Q}\) is defined as
\[
W_{d_v}(\mathsf{P},\mathsf{Q}) := \inf_{\rho \in \Pi(\mathsf{P}, \mathsf{Q})} \sum_{s, s' \in \mathcal{S}} \rho(s, s') \cdot d_v(s, s') = \sup_{\|f\|_{\text{Lip}(\tilde{d})} \le 1} \left| \mathbb{E}_\mathsf{P}[f] - \mathbb{E}_\mathsf{Q}[f] \right|,
\]
\end{definition}

\section{Cramér–Rao Bounds with Side Information}\label{Appendix: Cramer-Rao}
In this section, we examine the Cramér–Rao Bound (CRB) for the estimator $\kpn$ under the constrained estimation setting. The CRB provides a lower bound on the variance of a single parameter estimator or the covariance matrix of a vector with multiple parameters. In general, given an unbiased estimator $\hat{\zeta}$ of a $k$-element vector $\zeta$, the Fisher Information Matrix (FIM) is defined as $J(\zeta)=\mathbb{E}[\nabla\log(L(\zeta))\nabla\log(L(\zeta))^{\top}]$, where $\nabla\log(L(\zeta))$ is the gradient of the log-likelihood function w.r.t. $\zeta$. The CRB then states that the covariance of the unbiased estimator is lower bounded by the inverse $J(\zeta)^{-1}$ of the FIM. Before presenting the main result of this subsection, we derive an expression for the FIM $J(\kpt)$ as stated in the next lemma.
\begin{lemma} \label{lemma: FIM}
    Given $n$ i.i.d. samples $\{x_1,\dots, x_n\}\sim \kpt $, where $x\in \{s_1,\dots, s_k \}$, with $\kp_{\textup{t},j}:= \Pr(x=s_j),  \kp_{\textup{t},j}\neq 0, \forall j\leq k$, the FIM is given by
    \begin{equation}
        J(\kpt) = n \diag \left(\frac{1}{\kp_{\textup{t},1}},\dots, \frac{1}{\kp_{\textup{t},k}}\right)\:.
    \end{equation}
\end{lemma}

We can readily state the main result of this subsection. 

\begin{theorem}\label{thm: Cramer-Rao}
     Let $\kpn$ be an unbiased estimator of $\kpt$ based on $n$ i.i.d. samples. 
     For each of the following constraints, the diagonal elements of the CRB matrix, $C_{ii} (\kpt)$, are given by:
     
    \begin{itemize}
        \item \textbf{Regular Probability Constraint}  (\:$\sum_{i}^{k}\kp_{\textup{t},i} = 1 $): \quad  $C_{ii}(\kpt) = \frac{1}{n}\kp_{\textup{t},i}(1-\kp_{\textup{t},i})$

\item \textbf{Moment constraints} (\:$\mathbb{E}_{\kpt}[\phi(x)] = \mu$, where $x\in\mathcal{S}$ and $\phi: \mathcal{S}\to \mathbb{R}^{m}$, and $\sum_{i}^{k}\kp_{\textup{t},i} = 1 $) :  
       \begin{align}
           C_{ii}(\kpt) = C^{R}_{ii} (\kpt)- \Delta_{i}
       \end{align}
       where
    \begin{itemize}
        \item  $C^{R}_{ii}$ is the $i^{th}$ diagonal element of the CRB matrix with regular probability constraints. \item $\Delta_{i} = \frac{(\kp_{\textup{t},i})^{2}}{n^2} (a_{i} - \overline{a}) \mathbb{S}^{-1} (a_{i} - \overline{a}) $,
        where $a_{i} = \phi(x_i) - \phi(x_k) \in \mathbb{R}^{m}, i=1,\dots,k$, $\overline{a} = \sum_{i=1}^{k} a_i \kp_{\textup{t},i}$, and  $\mathbb{S} = \frac{1}{n} \operatorname{Cov}(a),$ where $\operatorname{Cov(a)}$ is the covariance matrix of $a$. 
    \end{itemize}
    
    \end{itemize}

\end{theorem}

Theorem \ref{thm: Cramer-Rao} provides valuable insights into the quality of the estimator as additional constraints are imposed. Specifically, the diagonal elements of the CRB give lower bounds on the variances of $\widetilde{\mathsf P}_{n,i}^{s,a}, 1 \leq i \leq k$, where $\widehat{\mathsf P}_{i,n}^{s,a}$ is the $i^{th}$ element of $\kpn$. A smaller lower bound indicates the potential for a better minimum variance estimator. To quantify the CRB matrix, we can compute the trace of the CRB $C$, which provides a lower bound on the sum of the estimator variances. Formally, we have $\sum_{i=1}^{k} \operatorname{var}[\kpn] \geq \operatorname{trace}(C(\kpt))= \sum_{i=1}^k C_{ii}(\kpt)$. According to Theorem \ref{thm: Cramer-Rao}, we conclude that:
\begin{equation}\label{eqn: trace comp}
 \frac{1}{n}\sum_{i=1}^k \kp_{\text{t},i}(1-\kp_{\text{t},i})\geq 
    \sum_{i=1}^k \frac{\kp_{\text{t},i}(1-\kp_{\text{t},i})}{n} -  \frac{(\kp_{\text{t},i})^{2} (a_i - \overline{a})^\top \mathbb{S} (a_i - \overline{a})}{n^{2}}.
\end{equation}
Since $\mathbb{S}$ is positive definite (assuming $\operatorname{Cov}(a)$ is full rank), $\Delta_{i}\geq 0$. 

Equation~\eqref{eqn: trace comp} demonstrates that having more prior information about $\kpt$ potentially leads to a better estimate, given the smaller bound on the variance. We also provide empirical evidence of the information gain that the side information provides in Appendix \ref{Appendix: cramer_rao_exp}.

\section{Value-Aware Side Information.} \label{Appendix: Value-Aware}
In this section, we present an evaluation error bound under the value-aware Wasserstein distance side information, where the bound is presented in terms of the side information. Formally, the IBE under value-aware 1-Wasserstein distance constraint is given as 
\begin{align} \label{eqn: VAW_eqn}
    \kpn = \arg\max_{q\in\Delta(\mathcal S)}
\sum_{s'\in\mathcal S} n_j q(s') \quad \text{subject to} \quad W_{d_V}(q, \mathsf{P}_s^{s,a}) \le \beta_1,
\end{align}
where \(W_{d_V}(\cdot, \cdot)\) is Value-Aware 1-Wasserstein distance (see Definition \ref{def: VAW}), and \(n_j = N_{s,a}(j),\quad j\in \mathcal{S}\).\par 

The following Theorem provides a bound on the evaluation error in terms of the side information \(\beta_1\). 

\begin{theorem}[Bound under Value-Aware Wasserstein Constraint] \label{thm: error_bound_w}
Let \(W_{d_V}(\cdot, \cdot)\) be the Value-Aware 1-Wasserstein distance (see Definition~\ref{def: VAW}), and \( \kpn \) be IBE, given by~\eqref{eqn: VAW_eqn}. Suppose
$
W_{d_V}(\kps,\kpt) \leq \beta_1.
$
Then, we have
\[
\left\|V^{\pi_n}_{\Pt}-V^*_{\Pt} \right\|_{\infty} \leq \frac{4\gamma \beta_1}{1 - \gamma}.
\]
\end{theorem}
 Theorem \ref{thm: error_bound_w} presents a bound on the evaluation error in terms of the side information \(\beta_1\). In practice, the target domain value function $V^{\pi}_{\Pt}$ is unknown, so we consider using the value function under the source dynamics $\kps$ to define the ground metric. We define the source metric \(\tilde{d}(s, s')\) as \(\tilde{d}(s, s') := |V^{\pi}_{\mathsf{P}_s^{s,a}}(s) - V^{\pi}_{\mathsf{P}_s^{s,a}}(s')|\). Therefore, we solve
\begin{align} \label{eqn: SVAW_eqn}
   \kpn = \arg\max_{q\in\Delta(\mathcal S)}
\sum_{s'\in\mathcal S} n_j q(s') \quad \text{subject to} \quad W_{\tilde{d}}(q, \mathsf{P}_\mathrm{s}^{s,a}) \le \beta_1,  
\end{align}

where \( W_{\tilde{d}}(\cdot,\cdot)\) is a Value-Aware 1-Wasserstein distance with the source metric \(\tilde{d}(s, s')\). We also define the mismatch factor w.r.t the uncertainty set \(\mathcal{P} = \bigotimes_{(s,a)} \mathbb B_{\mathrm{TV}}\bigl(\kp,\,R\bigr)\) as,
\begin{align}\label{eqn: mismatch factor}
    L_{\mathcal{P}} := \sup_{s \ne s'} \frac{|V^{\pi}_{\mathcal P}(s) - V^{\pi}_{\mathcal P}(s')|}{|V^{\pi}_{\kps}(s) - V^{\pi}_{\kps}(s')|}.
\end{align}

\begin{theorem}[Deviation Bound with Source-Defined Metric]\label{thm: error_bound_sw}
 Let \(W_{\tilde{d}}(\cdot, \cdot)\) be the Value-Aware  1-Wasserstein distance defined w.r.t the source metric \(\tilde{d}\), and \( \kpn \) be the IBE, given by~\eqref{eqn: SVAW_eqn}. Suppose
$
W_{\tilde{d}}(\kps,\kpt) \leq \beta_1.
$
Then, the value error is bounded as,
\[
\left\|V^{\pi_n}_{\Pt}-V^*_{\Pt} \right\|_{\infty} \le \frac{4 \gamma L_{\mathcal{P}_n}}{1 - \gamma} \cdot \beta_1.
\]
where \(L_{\mathcal{P}_n}\) is the mismatch factor w.r.t the uncertainty set \(\mathcal{P}_n\). 
\end{theorem}
Theorem~\ref{thm: error_bound_sw} bounds the evaluation error in terms of the side information \(\beta_1\), with more practical constraints \(\Phi\) that do not involve the target domain value function \(V^{\pi}_{\Pt}\).

\section{Deriving Side Information from System Knowledge} \label{Appendix: side_info}
In many robotics/control settings, the next state is generated as
$s' = f(s,a,\theta)+\xi$, where $\theta$ collects physical parameters (e.g., masses, friction, damping, actuator gains) and $\xi$ is additive noise with fixed law. Calibration, manufacturer tolerances, and system identification routinely yield \emph{bounded parameter sets} $\|\theta_t-\theta_s\|\le \varepsilon$. For any fixed $(s,a)$, consider the two next-state laws $\mathsf P_{\theta_1}^{s,a}, \mathsf P_{\theta_2}^{s,a}$ induced by $\theta_1,\theta_2$. We obtain
\[
W_1\!\big(\mathsf P_{\theta_1}^{s,a},\mathsf P_{\theta_2}^{s,a}\big) \;\le\; \mathbb{E}\,\big\|f(s,a,\theta_1)-f(s,a,\theta_2)\big\|.
\]
If \(f\) is \(L_\theta\)-Lipschitz in \(\theta\) (uniformly in \((s,a)\)), then
\[
W_1\!\big(\mathsf P_{\theta_1}^{s,a},\mathsf P_{\theta_2}^{s,a}\big) \;\le\; L_\theta\,\|\theta_1-\theta_2\|
\quad\Rightarrow\quad
W_1\!\big(\mathsf P_{\theta_t}^{s,a},\mathsf P_{\theta_s}^{s,a}\big) \;\le\; L_\theta\,\varepsilon.
\]
On discrete spaces with minimum separation $m>0$ under the ground cost, Kantorovich--Rubinstein yields
\[
\|\mathsf P_{\theta_t}^{s,a}-\mathsf P_{\theta_s}^{s,a}\|_{\mathrm{TV}} \;\le\; \tfrac{1}{m}\,W_1\!\big(\mathsf P_{\theta_t}^{s,a},\mathsf P_{\theta_s}^{s,a}\big) \;\le\; \tfrac{L_\theta}{m}\,\varepsilon,
\]
providing a non-vacuous TV and Wasserstein radius for our Distance--IBE constraints. When physical parameter bounds are unavailable, the Wasserstein radius $d_{s,a}$ can alternatively be obtained from data via KL estimation. Specifically, on a state space of bounded diameter $D$, Pinsker-type inequalities give
\[
W_1\!\big(\mathsf P_{\theta_t}^{s,a},\mathsf P_{\theta_s}^{s,a}\big) \;\le\; D\sqrt{\tfrac{1}{2}\mathrm{KL}\!\big(\mathsf P_{\theta_t}^{s,a}\|\mathsf P_{\theta_s}^{s,a}\big)},
\]
so that a minimax-optimal KL estimate \cite{zhao2020minimax} directly yields a valid Wasserstein radius, connecting sample-based divergence estimation to our Distance--IBE framework.

\section{Explicit Dependence of \(\delta_n\) on different forms of the side information \(\Phi\)} \label{Appendix: Dependence on sideinfo}
\noindent\textbf{Overview.}
Our value guarantees are indexed by the deviation
\(\delta_n \triangleq \max_{s,a}\|\widehat{\mathsf P}_n^{s,a}-\mathsf P_{\mathrm t}^{s,a}\|_{\mathrm{TV}}\).
This appendix explains how different forms of side information \(\Phi\) enter the bounds by restricting the feasible set of the constrained estimator and thereby reducing the deviation term \(\delta_n\).

\paragraph{Distance-IBE (TV / Wasserstein constraints).}
Suppose the side information specifies a source-anchored radius \(d\), e.g.,
\(W_1(\mathsf P_{\mathrm t}^{s,a},\mathsf P_{\mathrm s}^{s,a}) \le d\) or
\(\|\mathsf P_{\mathrm t}^{s,a}-\mathsf P_{\mathrm s}^{s,a}\|_{\mathrm{TV}} \le d\),
and that the estimator \(\widehat{\mathsf P}_n^{s,a}\) is constrained to the same ball around the source kernel \(\mathsf P_{\mathrm s}^{s,a}\).
Then, by the triangle inequality,
\[
\delta_n
=
\max_{s,a}\|\widehat{\mathsf P}_n^{s,a}-\mathsf P_{\mathrm t}^{s,a}\|_{\mathrm{TV}}
\le
\max_{s,a}\|\mathsf P_{\mathrm t}^{s,a}-\mathsf P_{\mathrm s}^{s,a}\|_{\mathrm{TV}}
+
\max_{s,a}\|\widehat{\mathsf P}_n^{s,a}-\mathsf P_{\mathrm s}^{s,a}\|_{\mathrm{TV}}
\le 2d.
\]
This yields an explicit, non-vacuous pre-asymptotic bound on \(\delta_n\) in terms of the side-information radius \(d\).
When side information is expressed in \(W_1\) or KL, analogous bounds follow via standard inequalities such as Kantorovich--Rubinstein and Pinsker’s inequality (e.g., on discrete \(\mathcal S\), \(\mathrm{TV}\le \tfrac{1}{m}W_1\) and \(\mathrm{TV}\le \sqrt{\tfrac{1}{2}\mathrm{KL}}\)).

\paragraph{Moment-IBE.}
A constraint of the form \(\|\mu(q)-\mu(\kps)\|\le\beta_{s,a}\) for feature map \(\phi\) induces an integral probability metric–type control:
\(
\sup_{\|g\|\le 1}\big|\mathbb E_q[g(\phi)]-\mathbb E_{\kpt}[g(\phi)]\big|
\le
\beta_{s,a} + \|\mu(\kpt)-\mu(\kps)\|.
\)
When the value function \(V\) is Lipschitz in \(\phi\), this control translates into a TV or \(W_1\) bound, thereby reducing \(\delta_n\) relative to unconstrained maximum-likelihood estimation at the same sample size \(n\).

\paragraph{Density-IBE.}
Density-ratio constraints of the form \(0 \le q \le B\,\mathsf P_{\mathrm s}\) enforce support overlap and prevent off-support mass.
For bounded \(B\), such constraints yield uniform bounds on deviations of the form
\(|q^\top v - (\mathsf P_{\mathrm t})^\top v|\) for bounded value functions \(v\),
which again translate into TV control via dual norms, tightening \(\delta_n\).

\paragraph{LDS-IBE.}
Under a \(d_0\)-dimensional parameterization and a TV-Lipschitz mapping
\(\theta \mapsto \mathsf P_\theta^{s,a}\),
the constrained maximum-likelihood estimator admits a finite-sample deviation
\(R_n = \tilde O(d_0/\sqrt n)\),
implying \(\delta_n = \tilde O(d_0/\sqrt n)\).
This rate underlies the improved robust suboptimality gap established in Theorem~7.

\medskip
\noindent
Collectively, these instantiations show how different choices of side information \(\Phi\) enter the guarantees explicitly through \(\delta_n\), and hence directly affect both the training and evaluation errors as well as the robust suboptimality gap.

\section{Proofs }\label{Appendix: Proofs}

\subsection{Proof of Theorem \ref{thm: train error bound}}
\begin{proof}
Considering any state $s$, we have that
\begin{align}\label{eqn: 14}
 |V^{\pi_{n}}_{\Phn}(s) -V^{\pi^*}_{\Pt}(s)| \leq \max_\pi |V^{\pi}_{\Phn}(s) -V^{\pi}_{\Pt}(s)|,
\end{align}
where the inequality is from the fact that $|\max f -\max g| \leq \max |f-g|$. 

We denote the \((s,a)\)-uncertainty sets \(\mathbb B_{\mathrm{TV}}\bigl(\kpt,\,R\bigr)\) and  \(\mathbb B_{\mathrm{TV}}\bigl(\kpn,\,R\bigr)\) as $(\Pt)_{s}^a$ and $(\Phn)_{s}^{a}$, respectively. 

For any fixed policy $\pi$, we have 
\begin{align}\label{eq:11}
    |V^{\pi}_{\Phn}(s)-V^{\pi}_{\Pt}(s)|&\overset{(a)}{=}|r_\pi(s)+\gamma \sigma_{(\Phn)^{\pi(s)}_s}(V^{\pi}_{\Phn})-r_\pi(s)-\gamma \sigma_{(\Pt)_{s}^a}(V^{\pi}_{\Pt})|\nonumber\\
    &=|\gamma \sigma_{(\mathcal{P}_n)^{\pi(s)}_s}(V^{\pi}_{\Phn})-\gamma \sigma_{(\Pt)_{s}^a}(V^{\pi}_{\Pt})|\nonumber\\
    &=|\gamma \sigma_{(\mathcal{P}_n)^{\pi(s)}_s}(V^{\pi}_{\Phn})-\gamma \sigma_{(\Pt)_{s}^a}(V^{\pi}_{\Phn})+\gamma \sigma_{(\Pt)_{s}^a}(V^{\pi}_{\Phn})-\gamma \sigma_{(\Pt)_{s}^a}(V^{\pi}_{\Pt})|\nonumber\\
    & \overset{(b)}{\leq}\gamma \|V^{\pi}_{\Phn}-V^{\pi}_{\Pt}\|_{\infty}+
    \gamma |\sigma_{(\mathcal{P}_n)^{\pi(s)}_s} (V^{\pi}_{\Phn})- \sigma_{(\Pt)_{s}^a}(V^{\pi}_{\Phn})|,
\end{align}
where $(a)$ is from the robust Bellman equation, and $(b)$ is from the Lipschitz smoothness of $\sigma_\mathcal{P}(\cdot)$ \cite{wang2021online}. 

Consider the second term $|\sigma_{(\mathcal{P}_n)^{\pi(s)}_s} (V^{\pi}_{\Phn})- \sigma_{(\Pt)_{s}^a}(V^{\pi}_{\Phn})|$. Note that the support function can be solved by its dual form:
\begin{align}
    \sigma_\mathcal{P}(V)=\min_{\mathsf P\in\mathcal{P}} \mathsf P^\top V =\max_{\alpha\geq 0} \{\mathsf P_0^\top (V-\alpha)-R\textbf{Span}(V-\alpha) \},
\end{align}
where $\mathsf P_0$ is the center of $\mathcal{P}$, $R$ is the radius, and \(\textbf{Span}(V) = \max_i V(i) - \min_i V(i)\), for the vector \(V=[V(1),\dots,V(S)]\). Hence, \eqref{eqn: 14} can be further bounded as 
\begin{align}\label{eq:9}
    |V^{\pi_n}_{\Phn}(s)-V^{\pi^*}_{\Pt}(s)| 
    \leq \gamma \|V^{\pi}_{\Phn}-V^{\pi}_{\Pt}\|+\gamma |(\widehat{\mathsf P}_{n}^{s,\pi(s)}-\mathsf P_{\text{t}}^{s,\pi(s)})^\top V^{\pi}_{\Phn}  |\:,
\end{align}
which is from the fact that $|\max_x f(x)-\max_x g(x)| \leq \max_x |f(x)-g(x)|$. Note that \eqref{eq:9} holds for any $s$ and any policy $\pi$, thus 
\begin{align}\label{eq:99}
    \|V^{\pi_n}_{\Phn}-V^{\pi^*}_{\Pt}\|_{\infty}
    \leq \gamma \|V^{\pi}_{\Phn}-V^{\pi}_{\Pt}\|_{\infty}+\gamma \max_{(s,a)}\|(\kpn-\kp_{\text{t}})^\top V^{\pi}_{\Phn}  \|_{1}\:.
\end{align}
Then, recursively applying \eqref{eq:99}, and using the fact \(V^{\pi}_{\Phn}(s)\leq \frac{1}{1-\gamma}\), implies that 
\begin{align}
    \|V^{\pi_n}_{\Phn}-V^{\pi^*}_{\Pt}\|_{\infty}\leq \sum^\infty_{t=1}\frac{\gamma^t}{1-\gamma}\max_{(s,a)}\|\kpn-\kp_{\text{t}}\|_{1}=2 \max_{(s,a)}\frac{\|\kpn-\kp_{\text{t}}\|_{TV}}{(1-\gamma)^2}\:.
\end{align}
\end{proof}

\subsection{Proof of Theorem \ref{thm: eval error bound}}
\begin{proof}
      Denote the optimal policy w.r.t. the uncertainty set \(\mathcal{P}_n\) by $\pi_n$, i.e.,
\begin{align}
\pi_n=\arg\max_\pi V^{\pi}_{\mathcal{P}_n}.
\end{align}
    Then, we have that
\begin{align}
   \|V^{\pi_n}_{\Pt}-V^{\pi^*}_{\Pt}\|_{\infty}\leq  \|V^{\pi_n}_{\Pt}-V^{\pi_n}_{\mathcal{P}_n}\|_{\infty} + \|V^{\pi_n}_{\mathcal{P}_n}-V^{\pi^*}_{\Pt}\|_{\infty}
   \label{eq:eval_bnd}
\end{align}.

Following the same technique in the  proof as of Theorem
\ref{thm: train error bound}, we know that 

\begin{align}
 \|V^{\pi_n}_{\mathcal{P}_n}-V^{\pi^*}_{\Pt}\|_{\infty}\leq         
\max_\pi \|V^{\pi}_{\Phn} -V^{\pi}_{\Pt}\|_{\infty}  \leq 2\max_{(s,a)}\frac{\|\kpn-\kpt\|_{TV}}{(1-\gamma)^2}
\label{eq:term1_eval_bound}
\end{align}
 and

\begin{align}
 \|V^{\pi_n}_{\Pt}-V^{\pi_n}_{\mathcal{P}_n}\|_{\infty}&\leq         
\max_\pi \|V^{\pi}_{\Phn} -V^{\pi}_{\Pt}\|_{\infty}\nonumber\\
&\leq \sum^\infty_{t=1}\frac{\gamma^t}{1-\gamma}\max_{(s,a)}\|\kpn-\kpt\|_{1}\nonumber\\
&= 2\max_{(s,a)}\frac{\|\kpn-\kpt\|_{TV}}{(1-\gamma)^2}.
\label{eq:term2_eval_bound}
\end{align}
 Combining RHS of \eqref{eq:term1_eval_bound} and \eqref{eq:term2_eval_bound} completes the proof.

\end{proof}

\subsection{Proof of Corollary \ref{cor:consistency}}
 
\begin{proof}
 The proof follows directly from the error bounds provided in Theorem \ref{thm: train error bound} and \ref{thm: eval error bound}.     
\end{proof}

\subsection{Proof of Proposition \ref{thm: CMLE in TV}}
We denote the optimal vanilla IBE by \(\kphs\), while the contrained IBE is denoted by \(\kpn\) . 
We first establish the following lemma on the convergence of the vanilla IBE in total variation.
\begin{lemma}\label{thm: MLE in TV}
Let $\kphs$ be the vanilla IBE of $\kp_{\textup{t}}$, then
\begin{equation}
    \lim_{n \to \infty} \|\kphs - \kp_{\text{t}}\|_{TV} = 0.
\end{equation}
\end{lemma}

\begin{proof}
By the consistency of the MLE, for a sequence $x^n = (x_1, \ldots, x_n)$ drawn i.i.d. from the discrete distribution $\kpt$, the MLE estimate of $\kpt$ is given by
\begin{equation}\label{eqn: consist MLE}
\widetilde{\mathsf P}_{n,j}^{s,a} = \frac{n_j}{n}, \quad j=1, \ldots, k
\end{equation}
where $n_j := \sum_{i=1}^n \mathbf{1}(x_i = j)$ represents the number of times state $j$ is observed in the sample.

By the Strong Law of Large Numbers (SLLN), we have that 
\[
\frac{n_j}{n} = \frac{1}{n}  \sum_{i=1}^n \mathbf{1}(x_i = j) \xrightarrow{\text{a.s.}} \kp_{\text{t},j}, \quad \forall  j\in [k].
\]
By the almost sure convergence, $\Pr(w\in \Omega: \lim_{n\rightarrow\infty} \frac{n_j}{n} =\kp_{\text{t},j}) = 1$, where $\Omega$ is the sample space. Hence, for any $\epsilon > 0$, $\exists N_j$ such that 
\begin{align}
\left|\frac{n_j}{n} - \kp_{\text{t},j}\right|\leq \frac{\epsilon}{k}, \quad \forall n\geq N_j.
\end{align}
Let $N = \max(N_1, \ldots, N_k)$. It follows that 
\begin{align}
\sum_{j = 1}^k\left|\frac{n_j}{n} - \kp_{\text{t},j}\right|\leq \epsilon, \quad \forall n\geq N.
\end{align}
By \eqref{eqn: consist MLE},
$\sum_{j = 1}^k\left|\widetilde{\mathsf P}_{n,j}^{s,a} - \kp_{\text{t},j}\right|\leq \epsilon, \quad \forall n\geq N$. Hence, $\|\kphs-\kpt\|_{TV}\rightarrow 0$, which completes the proof.
\end{proof}

Now, we are ready to prove Proposition \ref{thm: CMLE in TV}.
\begin{proof}
    We prove the results for each IBE method listed in Table \ref{tab: estimation methods}. Both the vanilla IBE and the IBE programs have a unique solution for each 
$n$, as the objective is strictly concave and the constraints are convex. 
We use the results of Lemma~\ref{thm: MLE in TV}, which states that for any $\varepsilon > 0$, $\exists N_0(\varepsilon)$ such that $\|\kphs - \kpt\|_{TV} \leq \varepsilon$, $\forall n > N_0$, to prove each IBE case. 

\noindent\textbf{Distance IBE}.\

Let $d_{s,a}^0 = \|\kps - \kpt\|_{1}= 2 \|\kps - \kpt\|_{TV}$ .
Since $\kpt$ is an interior point of the constraint set, we can choose $\varepsilon \leq d_{s,a}- \frac{d_{s,a}^0}{2}$. Hence, for sufficiently large $n$, 
 \begin{align}
    \|\kphs - \kps\|_{TV} \overset{(a)}{\leq} \|\kphs - \kpt\|_{TV} + \|\kps - \kpt\|_{TV} 
       \overset{(b)}{\leq} \varepsilon + \frac{d_{s,a}^0}{2}  
     & \overset{(c)}{\leq} d_{s,a}
\end{align}
where (a) follows from the triangular inequality, (b) by Theorem \ref{thm: MLE in TV}, and (c) from the choice of $\varepsilon$. Hence, 
$\kphs$ satisfies the distance constraint. Therefore, by the uniqueness of $\kpn$ (the solution of Equation~\eqref{eqn: s,a -est}), it follows that $\kphs$ solves the distance IBE program. Consequently, $\kpn\rightarrow\kpt$ in total variation.    

\noindent\textbf{Density IBE}.\\
We first denote the density ratio between $\kpt$ and $\kps$ by $B_{s,a}^0$, i.e.
\begin{align}\label{eqn : density}
    B_{s,a}^0 = \frac{\kpt}{\kps}
\end{align}
By Theorem \ref{thm: MLE in TV}, each entry of $\kphs$ converges to the corresponding entry of $\kpt$. For sufficiently large $n$, we have  
\begin{align} \label{eqn: one+eps}
    \left|\frac{\widetilde{\mathsf P}_{n,i}^{s,a}}{\mathsf P_{\text{t},i}^{s,a}} -1 \right| \leq \varepsilon 
    \implies \frac{\widetilde{\mathsf P}_{n,i}^{s,a}}{\mathsf P_{\text{t},i}^{s,a}} \leq 1 + \varepsilon    
\end{align}
 where $\widetilde{\mathsf P}_{n,i}^{s,a}$ and $\mathsf P_{\text{t},i}^{s,a}$ are the $i^{th}$ element of $\kphs$ and $\kpt$, respectively. 
 Then, we get that 
 \begin{align}
     \frac{\widetilde{\mathsf P}_{n,i}^{s,a}}{\mathsf P_{\text{s},i}^{s,a}}   \overset{(a)}{\leq} B_{s,a}^{0,i} (1 + \varepsilon)
     \overset{(b)}{\leq} B_{s,a}^{i}
 \end{align}
where $B_{s,a}^{0,i}$, $B_{s,a}^{i}$ and $\mathsf P_{\text{s},i}^{s,a}$ are the $i^{th}$ entry of $B_{s,a}^{0}$, $B_{s,a}$, and $\mathsf{P}_{\text{s}}^{s,a}$, respectively. Here: (a) follows from Eqs.~\eqref{eqn : density} and \eqref{eqn: one+eps}, while (b) by the choice of $\varepsilon$. We conclude that $\kphs$ satisfies the Density IBE's constraint. Since $\kpn$ is a unique solution of Equation~\eqref{eqn: s,a -est}, $\kphs$ solves the density IBE. Thus, the result follows. 

\noindent\textbf{Moment IBE}.\\
We define $\beta^{0}_{s,a}$ as the absolute difference between the embedding means of  $\kps$ and $\kpt$, i.e.
\begin{align}\label{eqn : moment}
    \beta^{0}_{s,a} = |\mu({\kpt})- \mu({\kps})|\:.
\end{align}
Now, we prove that each entry of $|\mu(\kphs)- \mu(\kpt)|$ goes to zero as $n \to \infty$. Suppose $\phi$ maps the state space (of $K$ states) to a high-dimensional feature space of dimension $M$, and consider the matrix $A$ of features of size $M\times K$, i.e.
\begin{align}
    A = \left[ \begin{array}{c}  \textbf{a}_{1}  \\ \vdots \\ \textbf{a}_{M}  \end{array} \right] = \left[ \begin{array}{ccc} a_{1,1} & \cdots & a_{1,K} \\ \vdots & \ddots & \vdots \\ a_{M,1} &  \cdots & a_{M,K} \end{array} \right]
\end{align}
where $\phi(x_{i}) = (a_{1,i}, \dots, a_{M,i})^\top, x_{i}\in \mathcal{S},i\leq K.$ 
\begin{align}
    |\mu({\kphs})- \mu({\kpt})|
    &= |A(\kphs - \kpt ) |  \nonumber \\ 
    &= \left[ \begin{array}{c} |\sum_{i=1}^K a_{1,i} (\widetilde{\mathsf P}_{n,i}^{s,a} - \mathsf P_{\text{t},i}^{s,a} ) | \\ \vdots \\ |\sum_{i=1}^K a_{M,i} (\widetilde{\mathsf P}_{n,i}^{s,a} - \mathsf P_{\text{t},i}^{s,a} ) | \end{array} \right] \nonumber \\
    &\leq \left[ \begin{array}{c}  \|\textbf{a}_{1}^{\top}  (\kphs -\kpt )\|_{1}  \\ \vdots \\ \|\textbf{a}_{M}^{\top}  (\kphs -\kpt )\|_{1} \end{array} \right] \nonumber \\
    & \overset{(*)}{\leq } 
    \left[ \begin{array}{c}  \|A\|_{\infty}  \|(\kphs -\kpt )\|_{1}  \\ \vdots \\ \|A\|_{\infty}  \|(\kphs -\kpt )\|_{1} \end{array} \right] \nonumber \\ &= 2 \|A\|_{\infty} \|\kphs - \kpt\|_{TV}  \mathbf{1}\:, 
\end{align} 
where $\mathbf{1}$ is the vector of all ones and (*) follows from H\"{o}lder's inequality. Hence, by Theorem \ref{thm: MLE in TV}, we have:
\begin{align}\label{eqn: moment contraints}
    \lim_{n\to \infty}|\mu({\kphs})- \mu({\kpt})| = 0\:.
\end{align}

For sufficiently large $n$, 
\begin{align}
|\mu({\kphs})- \mu({\kps})|&=  | A(\kphs - \kps )| \nonumber\\
&= |A(\kphs -\kpt +\kpt -\kps ) | \nonumber\\
&\overset{(a)}{\leq} |A (\kphs -\kpt)| + | A (\kpt - \kps) | \nonumber\\
&\overset{(b)}{\leq}\varepsilon + \beta^{0}_{s,a}  \nonumber \\
&\overset{(c)}{\leq} \beta_{s,a}
\end{align}
where (a) follows by the triangular inequality, (b) by using Eqs.~\eqref{eqn : moment} and \eqref{eqn: moment contraints}, and (c) by the choice $\varepsilon$. Then, $\kpn$ satisfies the moment constraints. The result follows by the uniqueness of $\kpn$.
 
\end{proof}

\subsection{Proof of Lemma \ref{lemma:radius}}
\begin{proof}
  
Define the constrained MLE 
    \[
    \widehat{\theta} = \arg\max_{\theta\in \Theta_0} \mathcal{L}(\theta)
    \]
    where $\mathcal{L}$ is the log-likelihood function under $\mathsf{P}_{\theta}$. By the consistency of the MLE, we have with probability at least $1-\delta_1$ that
    \[\|\widehat{\theta} - \theta^{*}\|\leq C_{0}\sqrt{\frac{\mathsf d_0}{n}}\]
    for some constant $C_0$.
  By Assumption \ref{assumption: Lipschitzness}, we have
  \[
  \|\mathsf{P}_{\widehat{\theta}} -\kpt\|_{TV} \leq \|\widehat{\theta} -\theta^{*} \|
  \leq L C_0\sqrt{\frac{\mathsf d_0}{n}}=: R_{n}\:.\] 

  This ensures $\mathsf{P}_{\mathrm t}\in \Phn(R_n)$ with probability at least $1-\delta_1$.  
\end{proof}

\subsection{Proof of Theorem \ref{thm: optgap}}

\begin{proof}
 The proof follows from Lemma~\ref{lemma:radius} and Theorem~\ref{thm: eval error bound}.
   
\end{proof}

\subsection{Proof of Theorem \ref{thm: error_bound_w}}

\begin{proof}
   Denote the optimal policy w.r.t the uncertainty set \(\mathcal{P}_n\) by $\pi_n$, i.e.,
\begin{align}
\pi_n=\arg\max_\pi V^{\pi}_{\mathcal{P}_n}.
\end{align}
    Then, we have that
\begin{align}\label{eq:9new}
   \|V^{\pi_n}_{\Pt}-V^{\pi^*}_{\Pt}\|_{\infty}\leq  \|V^{\pi_n}_{\Pt}-V^{\pi_n}_{\mathcal{P}_n}\|_{\infty} + \|V^{\pi_n}_{\mathcal{P}_n}-V^{\pi^*}_{\Pt}\|_{\infty}.
\end{align}

For the first term on the RHS of \eqref{eq:9new}, we follow the proof of Theorem \ref{thm: train error bound}. Therefore, we have for any \(s\in\mathcal{S}\):
\begin{align} 
    |V^{\pi_n}_{\Phn}(s)-V^{\pi^*}_{\Pt}(s)|\leq \max_\pi |V^{\pi}_{\Phn}(s) -V^{\pi}_{\Pt}(s)|. 
\end{align}

For any policy \(\pi\), we have
\begin{align}
    |V^{\pi}_{\Phn}(s) -V^{\pi}_{\Pt}(s)|\leq \gamma \|V^{\pi}_{\Phn}-V^{\pi}_{\Pt}\|+\gamma |(\widehat{\mathsf P}_{n}^{s,\pi(s)}-\mathsf P_{\text{t}}^{s,\pi(s)})^\top V^{\pi}_{\Phn}  |\:,
\end{align}

By Kantorovich duality, given two distributions \( \mathsf{P}, \mathsf{Q} \in \Delta(\mathcal{S}) \), the Wasserstein-1 distance induced by \( d_V \) is defined as
\[
W_{d_V}(\mathsf{P}, \mathsf{Q}) := \sup_{\|f\|_{\text{Lip}(d_V)} \leq 1} \left| \mathbb{E}_{\mathsf{P}}[f] - \mathbb{E}_{\mathsf{Q}}[f] \right|,
\]
where
\[
\|f\|_{\text{Lip}(d_V)} := \sup_{s \neq s'} \frac{|f(s) - f(s')|}{d_V(s, s')}.
\]
Since \( V^{\pi}_{\Phn} \) is 1-Lipschitz with respect to \( d_V \), we have
\begin{align} \label{eqn: lip}
    |(\widehat{\mathsf P}_{n}^{s,\pi(s)}-\mathsf P_{\text{t}}^{s,\pi(s)})^\top V^{\pi}_{\Phn}  | \leq W_{d_V}(\widehat{\mathsf P}_{n}^{s,\pi(s)}, \mathsf P_{\text{t}}^{s,\pi(s)}).
\end{align}

By the triangle inequality, we obtain,
\[
W_{d_V}(\widehat{\mathsf P}_{n}^{s,\pi(s)}, \mathsf P_{\text{t}}^{s,\pi(s)}) \leq W_{d_V}(\widehat{\mathsf P}_{n}^{s,\pi(s)}, \mathsf P_{\text{s}}^{s,\pi(s)}) + W_{d_V}(\mathsf P_{\text{s}}^{s,\pi(s)},\mathsf P_{\text{t}}^{s,\pi(s)}) \leq 2\beta_1.
\]
Therefore, we have, 
\[|V^{\pi}_{\Phn}(s) -V^{\pi}_{\Pt}(s)|
    \leq \gamma \|V^{\pi}_{\Phn}-V^{\pi}_{\Pt}\|_{\infty}+2\gamma\beta_1.\]
Applying the above recursively, we get 
\[
\|V^{\pi_n}_{\Phn}-V^{\pi^*}_{\Pt}\|_{\infty} 
    \leq  \frac{2\beta_1\gamma}{(1-\gamma)}.\]

Also, we have
\[\|V^{\pi_n}_{\Pt}-V^{\pi_n}_{\mathcal{P}_n}\|_{\infty} \leq         
\max_\pi \|V^{\pi}_{\Phn} -V^{\pi}_{\Pt}\|_{\infty} \leq \frac{2\beta_1\gamma}{(1-\gamma)}. \]

Combining the RHS of both terms of~\eqref{eq:9new} completes the proof. 
 
\end{proof}
  
\subsection{Proof of Theorem \ref{thm: error_bound_sw}}
\begin{proof}
Repeat the proof of Theorem~\ref{thm: error_bound_w}. The key change is that $V^{\pi}_{\Phn}$ may not be 1-Lipschitz under $\tilde{d}$, but it is $L_{\mathcal{P}_n}$-Lipschitz, i.e.,
\[| V^{\pi}_{\Phn} (s) -  V^{\pi}_{\Phn} (s')  |\le L_{\mathcal{P}_n} \cdot \tilde{d}(s, s'),\]\:

 Therefore,
\[
 |(\widehat{\mathsf P}_{n}^{s,\pi(s)}-\mathsf P_{\text{t}}^{s,\pi(s)})^\top V^{\pi}_{\Phn}  | \leq L_{\mathcal{P}_n} W_{\tilde{d}}(\widehat{\mathsf P}_{n}^{s,\pi(s)}, \mathsf P_{\text{t}}^{s,\pi(s)}).
\]
The rest of the proof proceeds identically.  
\end{proof}

\subsection{Proof of Lemma \ref{lemma: FIM}}
\begin{proof}
    To simplify notation, we denote $\kpt$ by $\kpp$.
Let $S$ be a random vector that takes one of the values from $\{e_1, \ldots, e_k\}$, where $e_j, j = 1,\ldots, k$, is a vector with a 1 in the $j$-th position and 0 elsewhere (one-hot encoding). The entries are denoted $s_j$, where $\Pr(s_j = 1) = \Pr(S = e_j) = \kpp_{j}$, so the vector $\kpp = [\kpp_1, \ldots, \kpp_k]^T$, and $\sum_{j=1}^k \kpp_{j} = 1$. To compute the Fisher Information Matrix (FIM), we define the likelihood
\[ \Pr(S|\kpp) = \prod_{j=1}^k {\kpp_{j}}^{s_j}. \]
Hence,
\[ \log \Pr(s|\kpp) = \sum_{j=1}^k s_j \log \kpp_{j}. \]

Therefore, the FIM, $I_S(\kpp) = \mathbb{E}[\nabla \log \Pr(s|\kpp) \nabla \log \Pr(s|\kpp)^T]$, is given by:
\[ I_S(\kpp) = \mathbb{E}\left( \left[ \begin{array}{c} s_1/\kpp_1 \\ \vdots \\ s_k/\kpp_{k} \end{array} \right]\left[ \begin{array}{ccc} s_1/\kpp_1 & \cdots & s_k/\kpp_k \end{array} \right]\right). \]

This simplifies to
\[ I_S(\kpp) = \mathbb{E} \left[ \begin{array}{cccc} (s_1/\kpp_1)^2 & s_1s_2/(\kpp_1 \kpp_2) & \cdots & s_1s_k/(\kpp_1\kpp_{k}) \\ \vdots & \ddots & \ddots & \vdots \\ s_ks_1/(\kpp_{k}\kpp_1) & \cdots & \cdots & (s_k/\kpp_{k})^2 \end{array} \right]. \]

Now, $\mathbb{E}[(s_i/\kpp_{i})^2] = \frac{1}{\mathsf P_{i}^{'2}} \mathbb{E}[s_i^2] = 1/\mathsf P_{i}^{'2} (1 \cdot \kpp_{i} + 0) = 1/\kpp_{i}, i = 1, \ldots, k$ and $\mathbb{E}[(s_i s_j/\kpp_{i} \kpp_{j})] = (1/\kpp_{i} \kpp_{j}) \mathbb{E}[s_i s_j] = 0$, for $i\ne j$. Thus, 
\[ I_S(\kpp) = \text{diag}\left(\frac{1}{\kpp_1}, \ldots, \frac{1}{\kpp_k}\right). \]

If we observe $n$ i.i.d. vectors $S$, we get that $J(\kpp) = n I_S(\kpp)$.

\end{proof}

\subsection{Proof of Theorem \ref{thm: Cramer-Rao}}
We use the same notation simplifications as in the proof of Lemma \ref{lemma: FIM}. 
We start by bringing in the following lemma.
\begin{lemma} [Theroem 2, \cite{marzetta1993simple}]\label{lemma: CCRB}
    Let the $k$-parameter vector $\widehat{\theta}$ be a constrained unbiased estimator of the parameter $\theta$, with $M$ constrains $f(\theta)=0$ for some real valued function $f$. Then, the CRB is given by:
    \begin{equation}
        C(\theta) = J^{-1} - J^{-1} F(F^{T} J^{-1} F)^{-1}F^{T}J^{-1}. 
    \end{equation}
    where J is the FIM and F is $k \times M$ gradient matrix w.r.t $\theta$.
\end{lemma}

\begin{proof}
    \noindent\textbf{Regular probability constraint} ($\sum_{i}^{k}\kpp_{i} = 1 $):\par
     Due to the normalization constraint $\sum_{i}^{k}\kpp_{i} = 1 $, the probabilities $\kpp_{i}$ are dependent. To handle this, we express $\kpp_{k} = 1 - \sum_{i}^{k-1}\kpp_{i}$, and use a reduced parameter vector $\theta = [\kpp_{1}, \dots, \kpp_{k-1}]$. Computing the FIM, we get:
    \begin{align}
         J = n \left(\diag\Big(\Big[\frac{1}{\kpp_{1}}, \dots, \frac{1}{\kpp_{k-1}}\Big]\Big) + \frac{1}{\kpp_{k}}\mathbf{1}_{k-1}\mathbf{1}_{k-1}^{\top}\right)
     \end{align}
     where $\mathbf{1}_{k-1}$ is a vector of all ones of size $k-1$. Using Sherman-Morrison-Woodbury formula~\cite{Petersen2008}, we compute the the CRB matrix, i.e., $J^{-1}$:
     \begin{align}\label{eqn: RCRB}
        C(\kpp) = J^{-1} = \frac{1}{n} \left(\diag(\kpp) - \kpp(\kpp)^{\top} \right)\:.
     \end{align}

      \noindent\textbf{Moment constraint} (\:$\mathbb{E}_{\kpp}[\phi(x)] = \mu$, where $x\in\mathcal{S}$ and $\phi: \mathcal{S}\to \mathbb{R}^{m}$, and $\sum_{i}^{k}\kpp_{i} = 1 $): 
      By Lemma \ref{lemma: CCRB}, the CRB matrix $C$ is given by:
\begin{align}
    C(\kpp) = {C_{R}}^{-1}- {C_{R}}^{-1}F (F^{T} {C_{R}}^{-1} F)^{-1}F^\top {C_{R}}^{-1}\:.
\end{align}
Here, we use $C_{R}$ to denote the CRB in Equation~\eqref{eqn: RCRB}, incorporating only the regular probability constraint. 
 The constraints, $ f(\kpp) = \mathbb{E}_{\kpp}[\phi(x)] - \mu = 0 \in \mathbb{R}^{m}$, can be rewritten as:
\begin{align}
    f(\kpp) &= \sum_{i=1}^{k-1} \phi(x_i)\kpp_{i} + \phi(x_k)\kpp_{k}-\mu = 0 \nonumber \\
    &\overset{(*)}{=} \sum_{i=1}^{k-1} (\phi(x_i) -\phi(x_k))\kpp_{i} + \phi(x_k)-\mu = 0 \nonumber \\
    & = \sum_{i=1}^{k-1} a_{i}\kpp_{i} + b =0  \quad (\text{$m$ constraints})
\end{align}
 where (*) follows from the constraint $\kpp_{k} = 1-\sum_{i=1}^{k-1} \kpp_{i}$, $ a_{i} = \phi(x_i) -\phi(x_k) \in \mathbb{R}^{m}$, and $ b = \phi(x_k) -\mu\in \mathbb{R}^{m}$. 

Let the constraint set $f(\kpp) = [f_{1},\dots,f_{m}]$, then
 the gradient matrix F of size $m\times (k-1)$ is given by:
\begin{align}
      \left[\begin{array}{ccc}
                                  \frac{df_{1}}{d\kpp_1}   & \dots & \frac{df_{1}}{d\kpp_{k-1}} \\
                                    \vdots & \ddots &  \vdots \\ 
                                    \frac{df_{m}}{d\kpp_{1}}  & \dots &\frac{df_{m}}{d\kpp_{k-1}} \\
                               \end{array}\right] = 
                               \left[\begin{array}{ccc}
                                  a_{1,1}   & \dots &  a_{1,k-1}  \\
                                    \vdots & \ddots &  \vdots \\ 
                                    a_{m,1}  & \dots &a_{m,k-1} \\
                               \end{array}\right]
\end{align}
where $a_{i} = [a_{1,i}, \dots, a_{m,i}]$. 
Then,
\begin{align}
    \mathbb{S}\overset{\Delta}{=} F^{T} {C_{R}}^{-1} F &=  
    \frac{1}{n} \left( \sum_{i=1}^{k-1}a_{i}a_{i}^{\top} \kpp_{i} -   \left(\sum_{i=1}^{k-1}a_{i}\kpp_{i}\right) \left(\sum_{j=1}^{k-1}a_{j}\kpp_{j}\right)\right)\nonumber\\
    &= \frac{1}{n} \operatorname{Cov}(a)
\end{align}
where $a$ is a vector that takes the value $a_{i}$ with probability $\kpp_{i}$, and $\operatorname{Cov}(a)$ is its covariance matrix.
Also, $C_{R}^{-1}F = \frac{1}{n}(\diag(\kpp_{i}) - \kpp(\kpp)^{\top})F =\frac{1}{n} (\mathbb{P} - \kpp\overline{a}^{\top})$, 
 where $\mathbb{P}$ is the $(k-1) \times m$ matrix with rows $\kpp_{i} a_{i}$ and $\overline{a} = \sum_{i=1}^{k} a_i \kp_{\textup{t},i}$. 
 Thus, 
 \[C(\kpp) = {C_{R}}^{-1}- ({C_{R}}^{-1}F)\mathbb{S}^{-1}({C_{R}}^{-1}F)^\top\]
  and 
  \[
  \operatorname{Var}[\kpp_{i}] \geq ({C_{R}}^{-1})_{ii} - ({C_{R}}^{-1}F)_{i}^{\top} \mathbb{S}^{-1}(C_{R}^{-1}F)_{i}
  \]
 where $ ({C_{R}}^{-1})_{ii} = \frac{\kpp_{i} (1-\kpp_{i})}{n}$ and $({C_{R}}^{-1}F)_{i}^{\top} = \frac{\kpp_{i}}{n} (a_{i} - \overline{a})$ is the $i^{th}$ row of $({C_{R}}^{-1}F)$.
 Then, 
 \begin{align}
     \Delta_{i} &:= ({C_{R}}^{-1}F)_{i}^{\top} \mathbb{S}^{-1}(C_{R}^{-1}F)_{i} \nonumber\\ 
     &= \left(\frac{\kpp_{i}}{n} (a_{i} - \overline{a})\right)^{\top}\mathbb{S}^{-1}
     \left(\frac{\kpp_{i}}{n} (a_{i} - \overline{a})\right)\nonumber\\
     &= \frac{(\kpp_{i})^{2}}{n^{2}}(a_{i} - \overline{a})^{\top} \mathbb{S}^{-1} (a_{i} - \overline{a})\:.
 \end{align}
Since $\mathbb{S}$ is positive definite (assuming $\operatorname{Cov}(a)$ is full rank), $\Delta_{i}\geq 0$, indicating that the lower bound on the variance of the estimator is reduced with the incorporation  of the moment constraint. 

We remark that, if $\operatorname{Cov}(a)$ is not a full rank matrix, then $\mathbb{S}^{-1}$ can be replaced with the Moore-Penrose pseudoinverse $\mathbb{S}^\dagger$, i.e., $C(\kpp) = {C_{R}}^{-1}- ({C_{R}}^{-1}F)\mathbb{S}^\dagger({C^{R}}^{-1}F)^\top$. The matrix $\Delta := ({C_{R}}^{-1}F) \mathbb{S}^\dagger({C_{R}}^{-1}F)^\top$ is positive semi-definite (PSD), so 
\[C(\kpp) = {C_{R}}^{-1} - \Delta \preceq {C_{R}}^{-1}\:,
\]
so the constrained matrix is less than or equal to the unconstrained one in the PSD ordering, indicating that the reduction in the lower bound on the variance still holds. 

\end{proof}


\section{Algorithms } \label{Appendix:algorithms}
To estimate the optimal policy of the worst-case value function as in Equation~\eqref{eqn: robust_value}, we utilize the robust value iteration algorithm presented in Algorithm \ref{alg:valueiteration}, where $\sigma_{\mathcal{P}^a_s}(V)\triangleq \min_{\mathsf P\in\mathcal{P}^a_s}\mathsf P^\top V$ is the support function of $V$ on $\mathcal{P}^a_s$. Moreover, Algorithm \ref{alg:valueiteration} can also be used to solve for the optimal policy of the value function for the non-robust, if we set $\mathcal{P}^a_s = \{\kp\}$, in which case $\sigma_{\mathcal{P}^a_s}(V)= {\kp}^\top V$

\begin{algorithm}[H]
\caption{Robust Policy Estimation}
\label{alg:valueiteration}
\textbf{Input}: $\gamma, V_0(s)=0,\forall s , T$
\begin{algorithmic}[1] 
\FOR{$t=0,1,...,T-1$}
    \FOR{all $s \in \mathcal{S}$}
        \STATE $V_{t+1}(s) \leftarrow \underset{a \in \mathcal{A}}{\max} \left\{ r(s, a) + \gamma \sigma_{\mathcal{P}^a_s}(V_t) \right\}$
    \ENDFOR
\ENDFOR
\FOR{$s \in \mathcal{S}$}
    \STATE $\pi_T(s) \leftarrow \arg\max_{a \in \mathcal{A}} \left\{  r(s, a) + \gamma \sigma_{\mathcal{P}^a_s}(V_T) \right\}$
\ENDFOR
\STATE \textbf{return} $V_T, \pi_T$
\end{algorithmic}
\end{algorithm}

Based on a similar approach as in policy estimation, any policy $\pi$ can be evaluated on a general uncertainty set, as shown in Algorithm \ref{alg:evaluatoin}, where in the non-robust we set \(R=0\).  

\begin{algorithm}[H]
\caption{Robust Policy Evaluation}
\label{alg:evaluatoin}
\textbf{Input}: $\pi, \gamma, V_0(s)=0,\forall s ,T$ 
\begin{algorithmic}[1]
\FOR{$t=0,1,...,T-1$}
\FOR{all $s\in\mathcal{S}$}
\STATE {$V_{t+1}(s)\leftarrow \mathbb{E}_{\pi}[r(s,A)+\gamma \sigma_{\mathcal{P}^A_s}(V_t)]$}
\ENDFOR
\ENDFOR
\STATE \textbf{return} $V_{T}$
\end{algorithmic}

\end{algorithm}

\section{Description of Environments}\label{Appendix: Env description}

\subsection{Toy text environments}

\noindent\textbf{Frozen Lake environment}:
This environment is depicted in Figure \ref{fig: frozen_lake}. We consider a $4\times 4$ frozen lake grid, consisting of a start state, an end state, two hole states, and one prize state, where each square represents a state. 
The main goal is to traverse the frozen lake from the start state to the end state without falling into holes. The agent can take four actions: right, left, up and down. The agent loses `-1' if it steps on a hole, receives `+5' for stepping on the prize state, and `-0.04' for any other frozen state. Transitions to the next state occur with a certain probability upon taking an action.

\begin{table}[]
    \centering  
    \caption{A Description of the source and target domain transition kernels for toy-text environments. }
\scalebox{0.70}{  \begin{tabular}{||c|c||} \hline
       Transition kernel & Description \\ \hline 
       $\mathsf P^{s,a}_{\text{s}}(s_{intended})$  &  The agent moves to the state in the intended direction with probability $(1-\alpha) (r_s + 2 \frac{(1-r_s)}{\mathcal{|S|}})$\\ \hline
       
        $\mathsf P^{s,a}_{\text{s}}(s_{opposite})$ & The agent may move to the state opposite the intended direction with probability $\alpha (r_s + 2 \frac{(1-r_s)}{\mathcal{|S|}})$\\ \hline 
        
        $\mathsf P^{s,a}_{\text{s}}(s_{other})$ & The agent may move to all other states   with probability $ \frac{(1-r_s)}{\mathcal{|S|}}$ \\ \hline 

        $\mathsf P^{s,a}_{\text{t}}(s_{intended})$  &  The agent moves to the state in the intended direction with probability $\alpha (r_t + 2 \frac{(1-r_t)}{\mathcal{|S|}})$\\ \hline
        
        $\mathsf P^{s,a}_{\text{t}}(s_{opposite})$ & The agent may move to the state opposite the intended direction with probability $(1 - \alpha) (r_t + 2 \frac{(1-r_t)}{\mathcal{|S|}})$\\ \hline 
        
        $\mathsf P^{s,a}_{\text{t}}(s_{other})$ & The agent may move to all other states   with probability $ \frac{(1-r_t)}{\mathcal{|S|}}$ \\ \hline

    \end{tabular}}
  
    \label{tab:toy_text kernels}
\end{table}

\noindent\textbf{Cliff Walking environment}:
The goal of the agent in the cliff walking environment is to reach the end (prize) state, starting from the first state while avoiding the cliff. The environment consists of 38 states: a start state, an end state, and 36 other states. The agent incurs a penalty of `-100' if it steps on the cliff and receives `+50' upon reaching the end state. For all other states, the agent receives a reward of `-1'. A schematic of this environment is depicted in Figure \ref{fig: Cliff_walking}.  

\noindent\textbf{Taxi environment}:
In the taxi environment, the goal for the agent (taxi) is to navigate through the grid using movement actions (up, down, left, right), locate the passenger, pick up the passenger, and later drop off the passenger at the destination using pick-up and drop-off actions, respectively. In our experiment, we consider a $6\times 6$ grid with fixed pick-up and drop-off locations, as shown in Figure \ref{fig: Taxi}. The passenger can be located either at the pick-up or drop-off locations, or in the taxi, resulting in a total of 108 states and 6 possible actions. The agent is penalized `-1' for each movement action taken. Wrongful pick-up and drop-off actions incur a penalty of `-10'. The agent receives a reward of `+20' for successfully picking up the passenger from the designated location and `+50' for successfully dropping off the passenger, signifying task completion.

\begin{figure}
  \centering

  \begin{subfigure}{0.4\textwidth}
    \includegraphics[width=\linewidth]{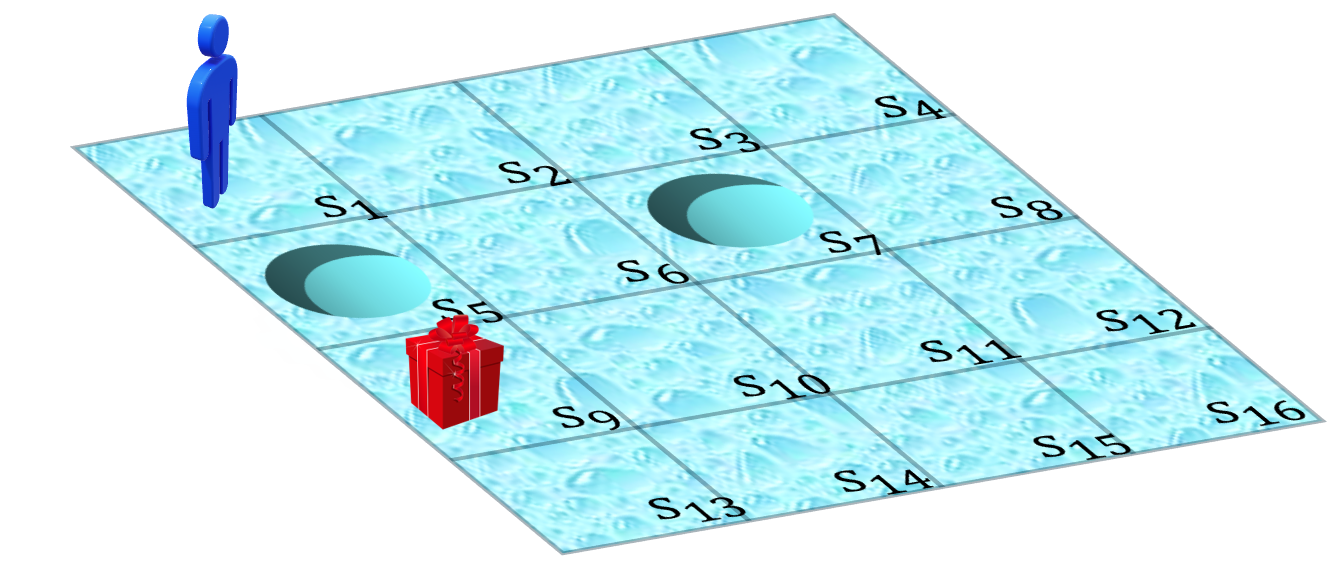}
    \caption{}
    \label{fig: frozen_lake}
  \end{subfigure}
  \hfill 
  \begin{subfigure}{0.4\textwidth}
    \includegraphics[width=\linewidth]{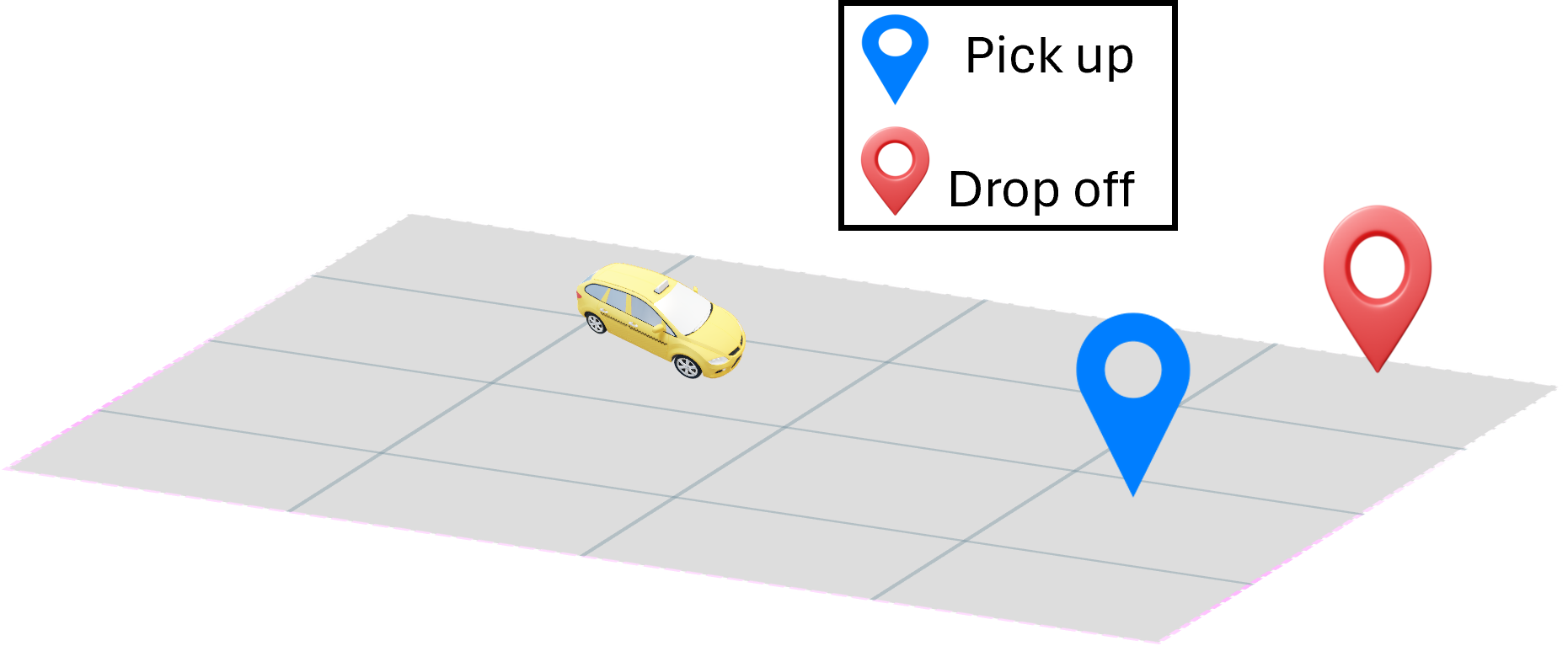}
    \caption{}
    \label{fig: Taxi}
  \end{subfigure}
    \hfill
  \begin{subfigure}{0.6\textwidth}
    \includegraphics[width=\linewidth]{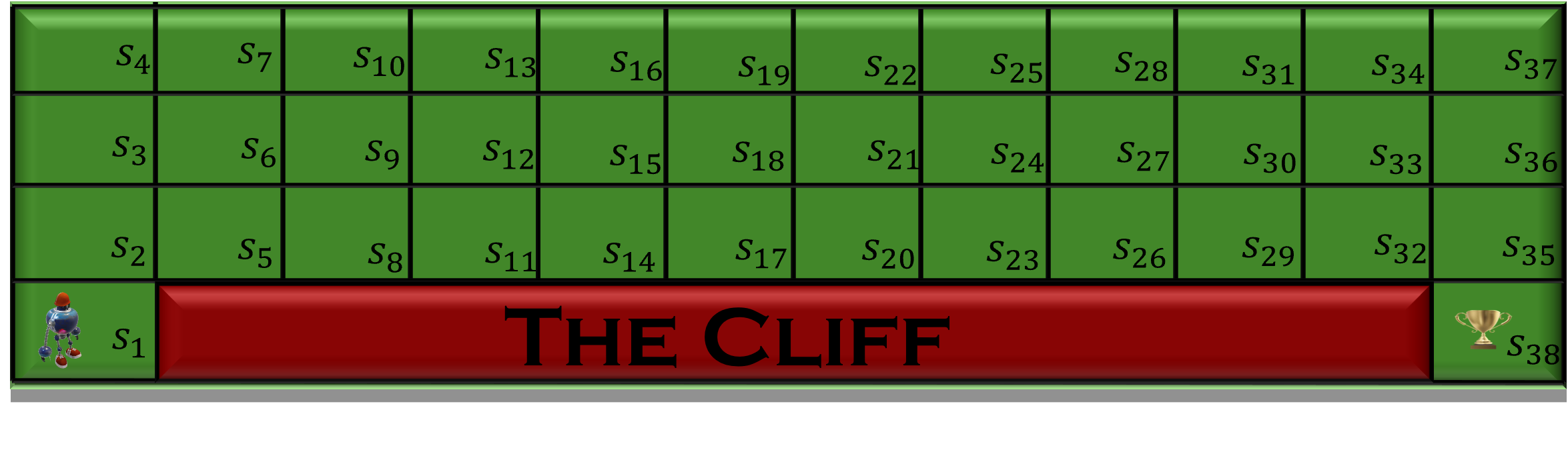}
    \caption{ }
    \label{fig: Cliff_walking}
  \end{subfigure}
 \caption{\small{A depiction OpenAI Gym toy text environments:\textbf{(a)} Frozen Lake environment  \textbf{(b)} Taxi environment\textbf{(c)} Cliff Walking environment }  }
  \label{fig: Toy-text env}
\end{figure}

A detailed description of the transition kernels for both the source and target domains is provided in Table \ref{tab:toy_text kernels}. We define $s_{intended}$ as the next state the agent is expected to move to when it takes action $a$ starting from state $s$. Conversely, $s_{opposite}$ refers to the state opposite to $s_{intended}$. For example, consider the frozen lake environment in the source domain. If the agent is at state $s_{2}$ and takes action $a$ = `right,' it will move to $s_{intended} = s_{3}$ with probability $\mathsf{P}^{s,a}_{\text{s}}(s_{intended})$, to $s_{opposite} = s_{1}$ with probability $\mathsf{P}^{s,a}_{\text{s}}(s_{opposite})$, and to any other state $s_{other}$ with probability $\mathsf{P}^{s,a}_{\text{s}}(s_{other})$. For the frozen lake environment, we use $r_s = 0.3$, $r_t = 0.8$, and $\alpha = 0.7$. For cliff walking, we set $r_s = 0.8$, $r_t = 0.2$, and $\alpha = 0.3$. For the taxi environment, we choose $r_s = 0.4$, $r_t = 0.8$, and $\alpha = 0.2$. To ensure validity, we normalize each $\mathsf{P}^{s,a}_{\text{s}}$ and $\mathsf{P}^{s,a}_{\text{t}}$ so that they sum up to 1.

\subsection{Classic control problems}
We consider three classical control problems from OpenAIGym: Cart pole, Acrobot, and Pendulum. We have made some modifications to the environments to suit our problem setup.\par

\noindent\textbf{Cart Pole environment}:
The goal is to balance a pole attached upright to the cart by applying forces to the left or right of the cart \cite{barto1983neuronlike}. The action space consists of two actions: right and left. The observation space is continuous and consists of a tuple of four quantities: Cart Position, Cart Velocity, Pole Angle, Pole Angular Velocity. Their corresponding ranges are $([-4.8,4.8], (-\infty,\infty), [-24^{o},24^{o}], (-\infty,\infty) ) $. We discretize the observation space, with each combination of values representing a distinct state. The new ranges for each quantity are $([-4.8,4.8], [-0.5,0.5], [-24^{o},24^{o}], [-5,5])$, and the number of discrete values is  $(4, 4, 5, 3)$, respectively. A reward of `25' is received if the pole angle is 0 and the cart velocity is between -0.17 and +0.17. An additional reward of `25' is granted if the cart position is between -1.6 and 1.6. 
When the pole angle is between $(-12^{o},12^{o})$ and the cart velocity is between -0.17 and +0.17, a reward of `10' is received. For all other cases, no rewards are received. The Cart Pole environment is depicted in Figure \ref{fig: Cart pole}. \par

\noindent\textbf{Cart Pole environment (LDS scenario)}: For each $(s,a)$, the source and target transition kernels are softmaxes over next–state features $\phi(s')\in\mathbb{R}^4$ (constructed from \textit{Cart Position, Cart Velocity, Pole Angle, Pole Angular Velocity}), the source and target transition kernels for each \((s,a)\in \mathcal{S}\times \mathcal{A}\) are modeled as follows: 
\[
\mathsf{P}_s(s'\mid s,a)=\frac{\exp\{\theta^{(s,a)}_s{}^\top\phi(s')\}}{\sum_{\bar s}\exp\{\theta^{(s,a)}_s{}^\top\phi(\bar s)\}},
\qquad
\mathsf{P}_t(s'\mid s,a)=\frac{\exp\{\theta^{(s,a)}_t{}^\top\phi(s')\}}{\sum_{\bar s}\exp\{\theta^{(s,a)}_t{}^\top\phi(\bar s)\}}.
\]
With $\theta^{(s,a)}_s=[\gamma_{sa};\,\alpha^{(s)}_{sa}]$ and $\theta^{(s,a)}_t=[\gamma_{sa};\,\alpha^{(t)}_{sa}]$ with the first $\mathsf{d}-\mathsf{d}_0$ coordinates shared ($\gamma_{sa}\in\mathbb{R}^{\mathsf{d}-\mathsf{d}_0}$) and the last $\mathsf{d}_0=2$ coordinates private ($\alpha^{(s)}_{sa},\alpha^{(t)}_{sa}\in\mathbb{R}^{2}$); hence the source–target shift in the logits (and thus in the kernels) lies in a $2$-dimensional subspace. 
\begin{table}
    \centering
        \caption{\small{A Description of the source and target domain transition kernels for classic control environments.} }
     
 \scalebox{0.85}{\begin{tabular}{||c|c||} \hline
    Transition kernel & Description \\ \hline
         $\mathsf P^{s,a}_{\text{s}}(s_{random_{1}})$  &  The agent moves to a random  state $s_{random_{1}}$ with probability $\alpha r_s$\\ \hline 
          $\mathsf P^{s,a}_{\text{s}}(s_{random_{2}})$& The agent moves to a random  state $s_{random_{1}}$ with probability $(1-\alpha) r_s$ \\ \hline
           $\mathsf P^{s,a}_{\text{s}}(s_{other})$& The agent may move to all other states with probability with probability $\frac{r_s}{|\mathcal{S}|} $. \\ \hline
           
          $\mathsf P^{s,a}_{\text{t}}(s_{random_{1}})$  &  The agent moves to a random  state $s_{random_{1}}$ with probability $(1-\alpha) r_t$\\ \hline 
          $\mathsf P^{s,a}_{\text{t}}(s_{random_{2}})$& The agent moves to a random  state $s_{random_{1}}$ with probability $\alpha r_t$ \\ \hline
          
           $\mathsf P^{s,a}_{\text{t}}(s_{other})$& The agent may move to all other states with probability with probability $\frac{r_t}{|\mathcal{S}|} $. \\ \hline
    \end{tabular}}   
    \label{tab:classic control kernels}
\end{table}

\noindent\textbf{Acrobot environment}:
The Acrobot problem, introduced in \cite{NIPS1995_8f1d4362}, features a chain with two links, one fixed and one free end, as depicted in Figure \ref{fig: Acrobot}.
The objective is to swing the free end of the chain above a certain height, starting from a downward hanging position. The joint between the two links is actuated, allowing for the application of torque to achieve the desired swing. The available actions involve applying a torque to the actuated joint, with options of $\{1, 0, -1\}$ Newton-meter (Nm). Observations in this environment consist of tuples of 6 quantities: $ (\cos\theta_{1}, \sin\theta_{1}, \cos\theta_{2}, \sin\theta_{2}, V_{\theta_{1}}, V_{\theta_{2}})$. Here, $\theta_{1}$ and $\theta_{2}$ represent the angles of the first link w.r.t. the normal direction and the angle of the second link relative to the first link respectively. The terms $V_{\theta_{1}}$ and $V_{\theta_{2}}$ denote the angular velocities of $\theta_{1}$ $\theta_{2}$, respectively.

For  the observation space, we use a tuple of four quantities $ (\cos\theta_{1}, \cos\theta_{2},  V_{\theta_{1}}, V_{\theta_{2}})$, with a range of \[([-1,1], [-1,1], [-4\pi,4\pi],[-9\pi, 9\pi])\], and the number of discrete values $(6, 6, 2, 2)$, respectively. For the reward, we use the quantity $\phi = -\cos(\theta_{1})-\cos(\theta_{1} + \theta_{2})$ to determine different reward levels. Specifically, for $\phi$ falling within predefined ranges ($\phi \geq 1$,  $0.5\leq \phi \leq 1$,  $0.25\leq \phi \leq 0.5$, $0\leq \phi \leq 0.25$), rewards of $20$, $15$, $10$, $5$ are granted, respectively.

\noindent\textbf{Pendulum environment}:
The Pendulum problem is another classic control environment that focuses on the dynamic control of an inverted pendulum. Illustrated in Figure \ref{fig: Pend}, the system comprises a pendulum affixed at one end to a fixed point. The objective is to apply torque and stabilize the pendulum's center of gravity directly above the pivot, swinging it into an upright position. The action space is continuous, within the range of $[-2, 2]$ Nm. Observations consist of arrays of 3 quantities representing the $x$ and $y$ coordinates of the free end of the pendulum, and the angular velocity, with ranges of $[-1,1], [-1,1]$, and $[-8,8]$, respectively. In our experiment, we discretize the action space into five torque levels $ \{-2, -1, 0, 1, 2\}$ Nm. The observation space is segmented into 240 unique states by discretizing the pendulum angle $[-\pi,\pi]$ rad into $12$ segments and the angular velocity ranging $[-10, 10]$ rad/s into $20$ segments. Rewards are structured to prioritize stabilizing the pendulum near the upright position. Maintaining the pendulum within $[-1,1] $ of vertical and sustaining a low angular velocity of $[-1,1]$ earns 100 points. A reward of 50 points is granted for keeping the angle within $[-0.5,0.5] $, and $10$ points are awarded for all other states.\par

Table \ref{tab:classic control kernels} provides a description of the transition kernels for both the source and target domains in the classic control problems. For each $(s, a)$ pair, we generate two random states, $s_{random_{1}}$ and $s_{random_{2}}$. Taking the source domain as an example, the agent transitions to $s_{random_{1}}$ and $s_{random_{2}}$ with probabilities $\mathsf{P}^{s,a}_{\text{s}}(s_{random_{1}})$ and $\mathsf{P}^{s,a}_{\text{s}}(s_{random_{2}})$, respectively. For all other states, $s_{other}$, the agent transitions with probability $\mathsf{P}^{s,a}_{\text{s}}(s_{other})$. 
We set $r_s = 0.6$, $r_t = 0.7$, and $\alpha = 0.2$ for all environments. To ensure validity, each $\mathsf{P}^{s,a}_{\text{s}}$ and $\mathsf{P}^{s,a}_{\text{t}}$ is normalized to sum to 1.\par

\begin{figure}
  \centering
  \begin{subfigure}{0.25\textwidth}
    \includegraphics[width=\linewidth]{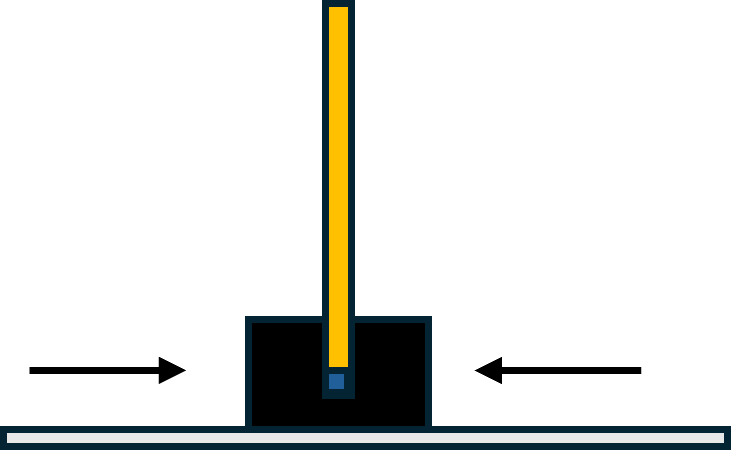}
    \caption{}
    \label{fig: Cart pole}
  \end{subfigure}
  \hspace{1cm}
  \begin{subfigure}{0.15\textwidth}
    \includegraphics[width=\linewidth]{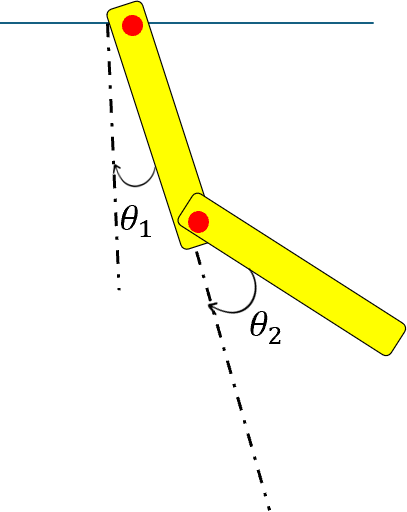}
    \caption{}
    \label{fig: Acrobot}
  \end{subfigure}
    \hspace{1cm}
  \begin{subfigure}{0.15\textwidth}
    \includegraphics[width=\linewidth]{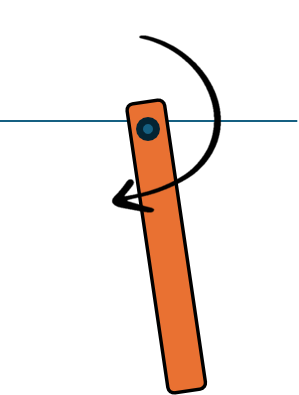}
    \caption{ }
    \label{fig: Pend}
  \end{subfigure}
 \caption{\small{A depiction OpenAI Gym classic control problems: \textbf{(a)} Cart Pole. \textbf{(b)} Acrobot. \textbf{(c)} Pendulum.   }  }
  \label{fig: classic problems}
  
\end{figure}
\section{Additional Experiments }\label{Appendix: Additional exp}
 \subsection{Convergence of the estimators}\label{Appendix: convergence_exp}

  We experimentally verify the convergence of our estimators to the target domain transition kernels, as theoretically established in Proposition \ref{thm: CMLE in TV}. Figure \ref{fig: overall_distance} shows the average distance over all $(s,a)$ pairs between our IBE estimates $\kph$  and the true target transition  kernels $\kpt$, i.e., $\frac{1}{|\mathcal{S}||\mathcal{A}|}\sum_{(s,a)}\|\widehat{\mathsf{P}}^{s,a} - \mathsf{P}^{s,a}_\text{t}\|_{TV}$, averaged over $20$ runs, versus the sample size for different toy text environments, including Frozen Lake, Cliff Walking, and Taxi. We also report the average distance between the source and target domain transition kernels. The results indicate that the IBE converges asymptotically to the target domain kernels as $N\to \infty$.  \par 

The results demonstrate that our estimators can converge to the true target transition kernels under different types of side of information.

\begin{figure}
  \centering
  \begin{subfigure}{0.32\textwidth}
    \includegraphics[width=\linewidth]{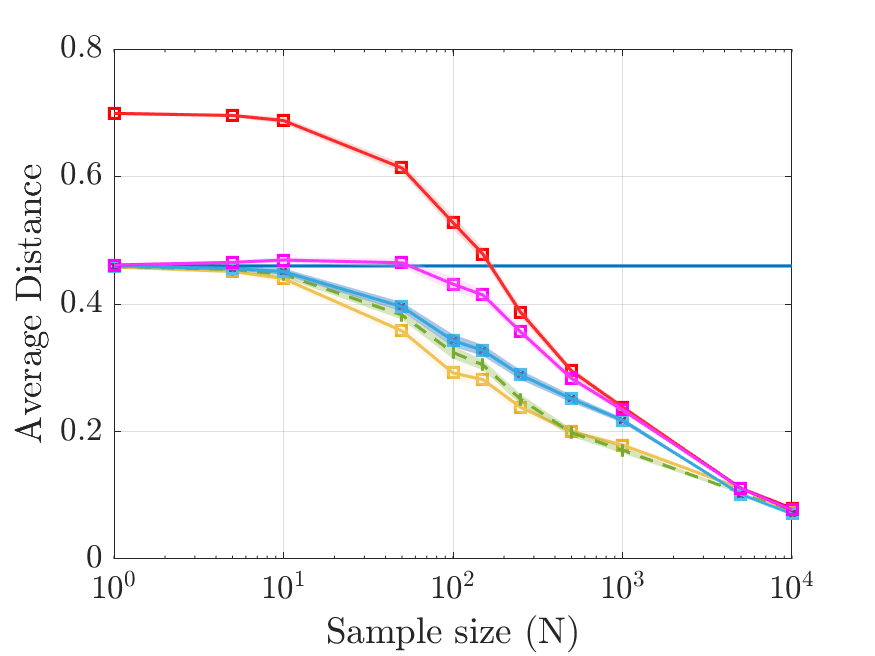}
    \caption{Frozen Lake.}
    \label{fig: frozen_lake_distance}
  \end{subfigure}
    \hfill
  \begin{subfigure}{0.32\textwidth}
    \includegraphics[width=\linewidth]{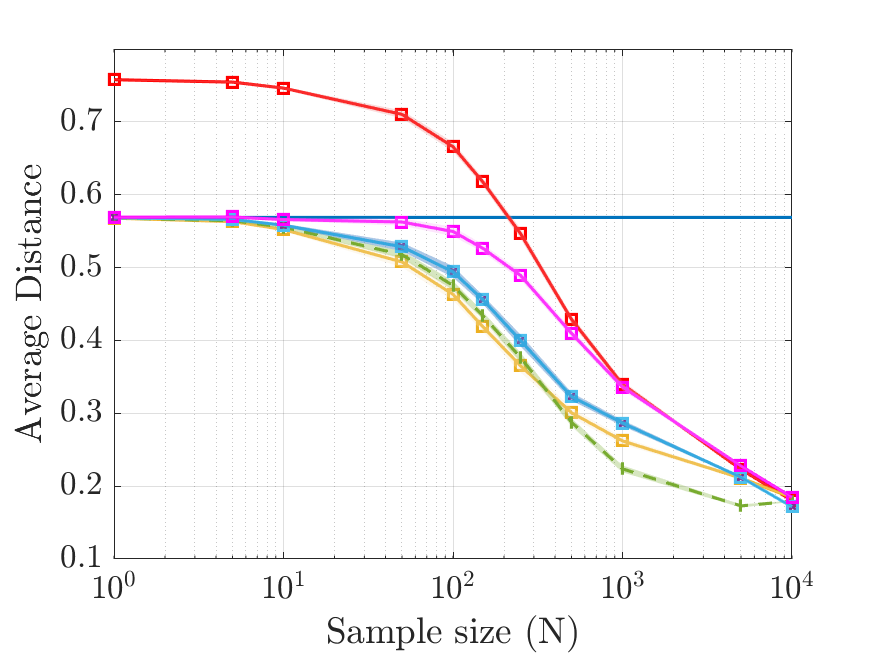}
    \caption{ Cliff Walking. }
    \label{fig: Walking_cliff_ditance}
  \end{subfigure}
    \hfill
  \begin{subfigure}{0.32\textwidth}
    \includegraphics[width=\linewidth]{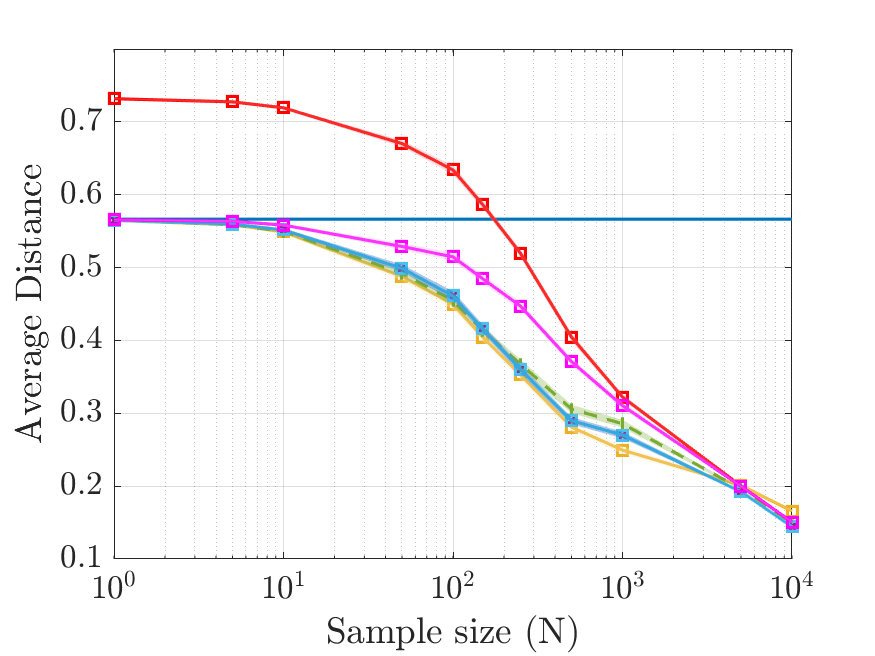}
    \caption{Taxi.}
    \label{fig: Taxi_ditance}
  \end{subfigure}
  \includegraphics[width=.85\textwidth]{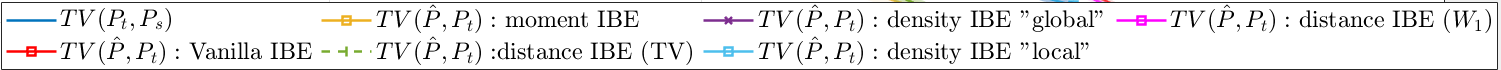}
 \caption{\small{Average total variation distance between the estimated and the target domain transition kernels over $20$ runs as a function of sample size for the toy text environments.}}
  \label{fig: overall_distance}
\end{figure}

\subsection{ Non-robust setting }\label{Appendix: non_robust_exp}
In this section, we provide the performance results for the omitted environments for the non-robust setting. Under this setting, we estimation the IBE policies using Algorithm \ref{alg:valueiteration}, where \(R=0\). Then, the estimated policies, including the IBE policies and the state-of-the-art (SOTA) ones, are evaluated using Algorithm~\ref{alg:evaluatoin}, using the same radius. 
Figures \ref{fig: cliff_walking_non_robust}, \ref{fig: Taxi_non_robust}, \ref{fig: Acrobot_non_robust}, \ref{fig: Frozen_lake_non_robust}, and \ref{fig: Pend_non_robust} demonstrate the target task performance, the average value function, i.e., $\frac{1}{|\mathcal{S}|}\sum_{s\in \mathcal{S}} V^{\pi}(s)$, as a function of the sample size $N$ for the non-robust scenario. Each figure shows the IBE with different side information (left panel) and a comparison with SOTA methods (right panel).  \par

\begin{figure}
  \centering
  \begin{subfigure}{0.35\textwidth}
    \includegraphics[width=\linewidth]{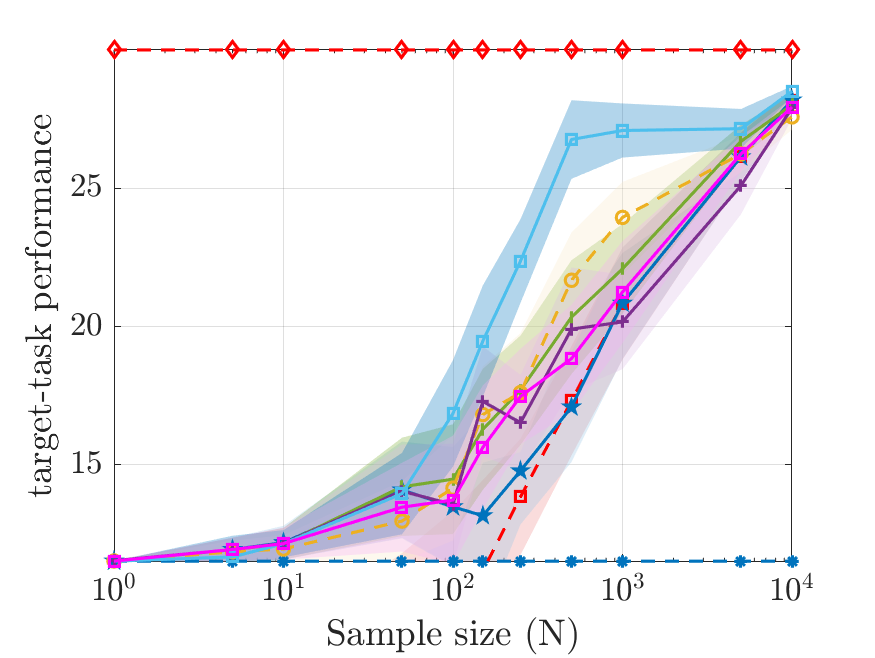}
    \caption{IBE comparison.}
    \label{fig: cliff_walking_IBE}
  \end{subfigure}
  \hspace{-5mm}
  \begin{subfigure}{0.35\textwidth}
    \includegraphics[width=\linewidth]{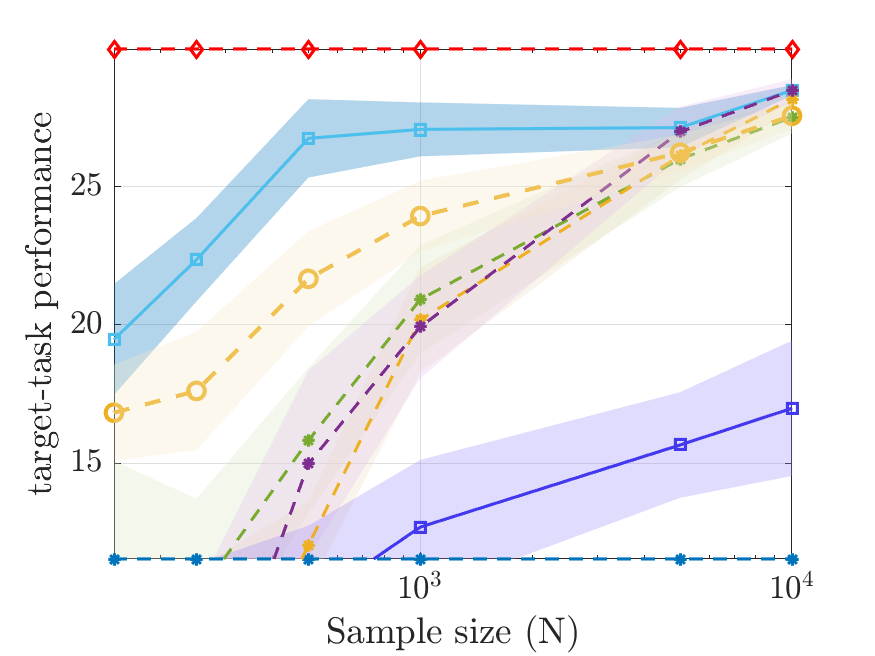}
    \caption{SOTA.}
    \label{fig: cliff_walking_STOA}
  \end{subfigure}
 
  \includegraphics[width=.85\textwidth]{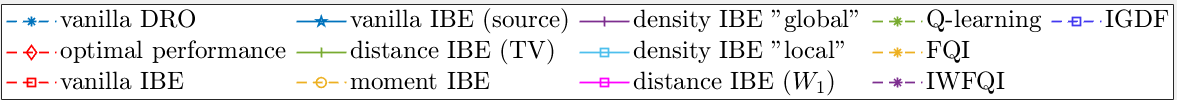}
 \caption{\small{Target domain performance for the non-robust setting averaged over $20$ runs as a function of sample size for Cliff Walking environment, and the corresponding $95\%$ confidence intervals.}}
  \label{fig: cliff_walking_non_robust}
\end{figure}

\begin{figure}
  \centering
  \begin{subfigure}{0.35\textwidth}
    \includegraphics[width=\linewidth]{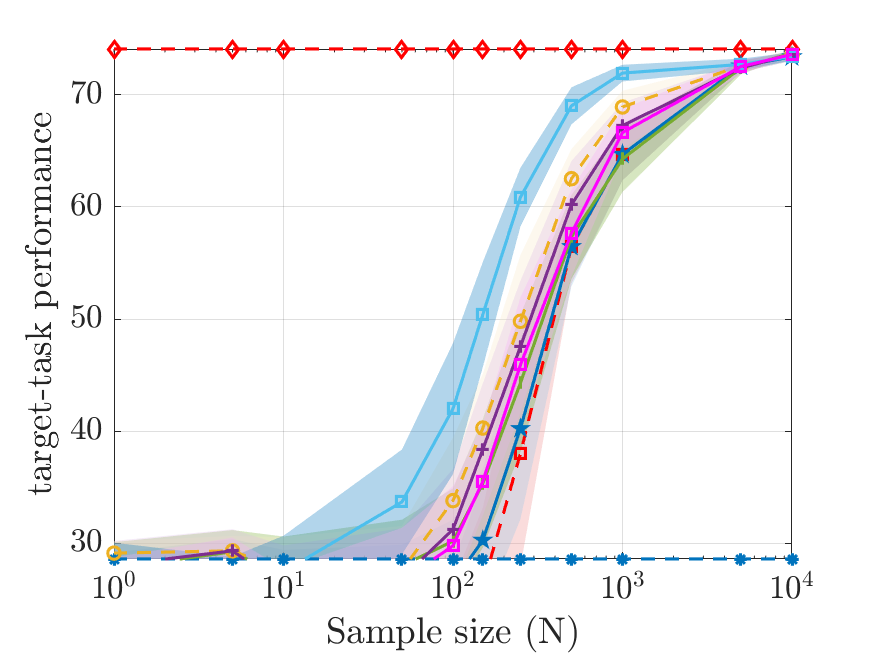}
    \caption{IBE comparison.}
    \label{fig: Taxi_IBE}
  \end{subfigure}
  \hspace{-5mm}
  \begin{subfigure}{0.35\textwidth}
    \includegraphics[width=\linewidth]{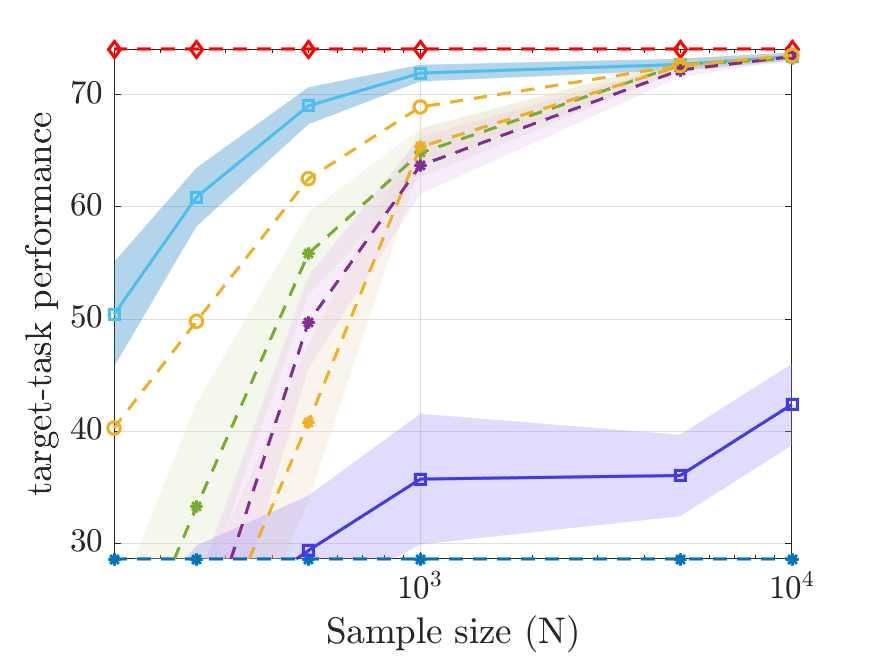}
    \caption{SOTA.}
    \label{fig: Taxi_STOA}
  \end{subfigure}
 
  \includegraphics[width=.85\textwidth]{Figures/NR_legends.PNG}
 \caption{\small{Target domain performance for the non-robust setting averaged over $20$ runs as a function of sample size  for Taxi environment, and the corresponding $95\%$ confidence intervals.}}
  \label{fig: Taxi_non_robust}
\end{figure}

\begin{figure}
  \centering
  \begin{subfigure}{0.35\textwidth}
    \includegraphics[width=\linewidth]{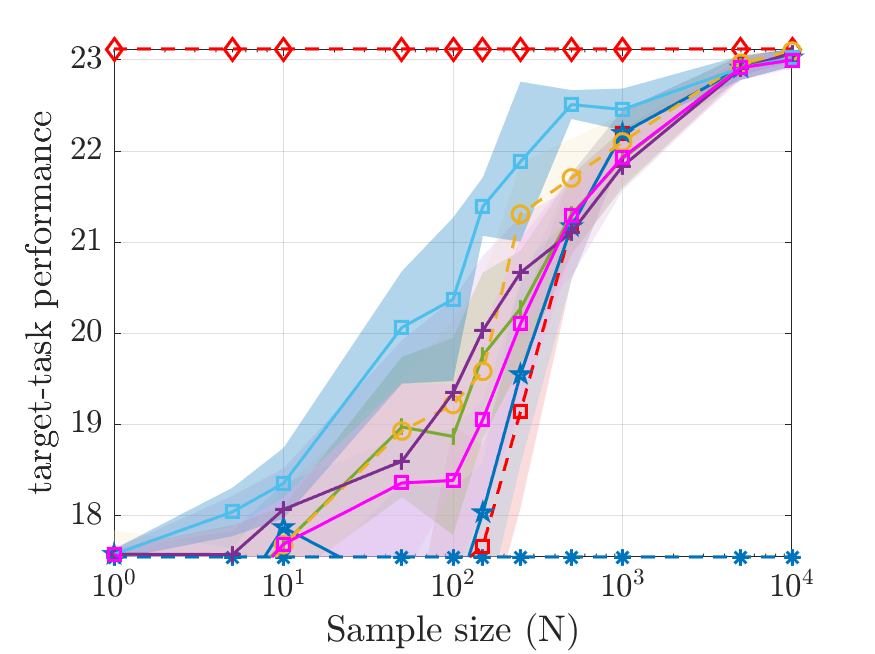}
    \caption{IBE comparison.}
    \label{fig: Cart_Pole_IBE}
  \end{subfigure}
  \hspace{-5mm}
  \begin{subfigure}{0.35\textwidth}
    \includegraphics[width=\linewidth]{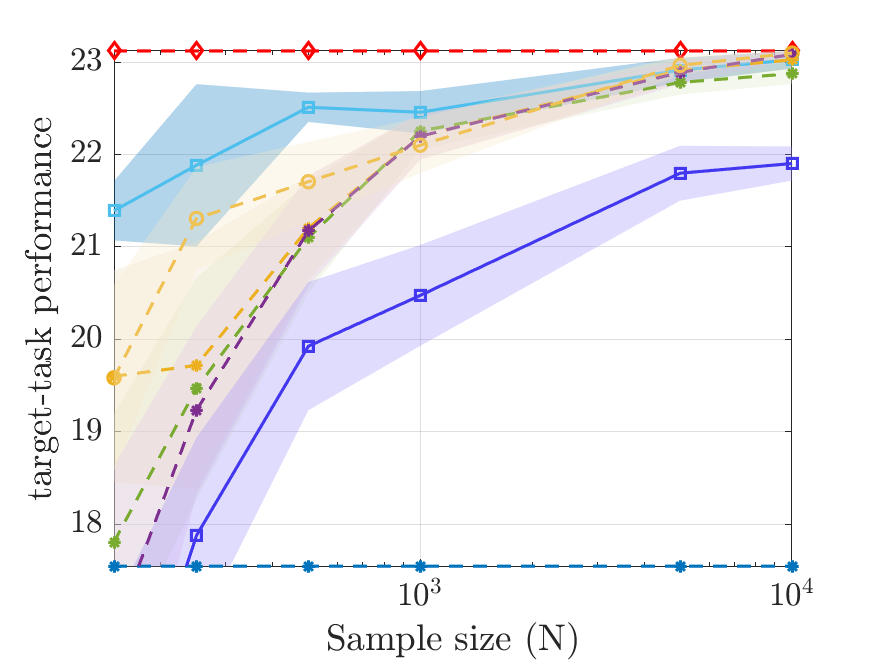}
    \caption{SOTA.}
    \label{fig: Cart_Pole_STOA}
  \end{subfigure}

  \includegraphics[width=.85\textwidth]{Figures/NR_legends.PNG}
 \caption{\small{Target domain performance for the non-robust setting averaged over $20$ runs as a function of sample size for Frozen Lake environment, and the corresponding $95\%$ confidence intervals.}}
  \label{fig: Frozen_lake_non_robust}
\end{figure}

\begin{figure}
  \centering
  \begin{subfigure}{0.35\textwidth}
    \includegraphics[width=\linewidth]{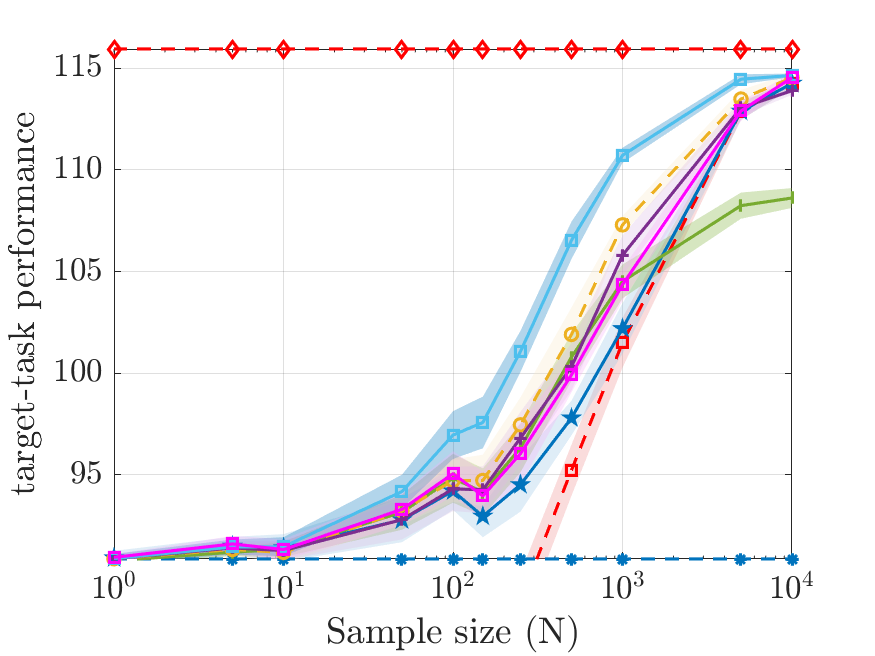}
    \caption{IBE comparison.}
    \label{fig: Acrobot_IBE}
  \end{subfigure}
  \hspace{-5mm}
  \begin{subfigure}{0.35\textwidth}
    \includegraphics[width=\linewidth]{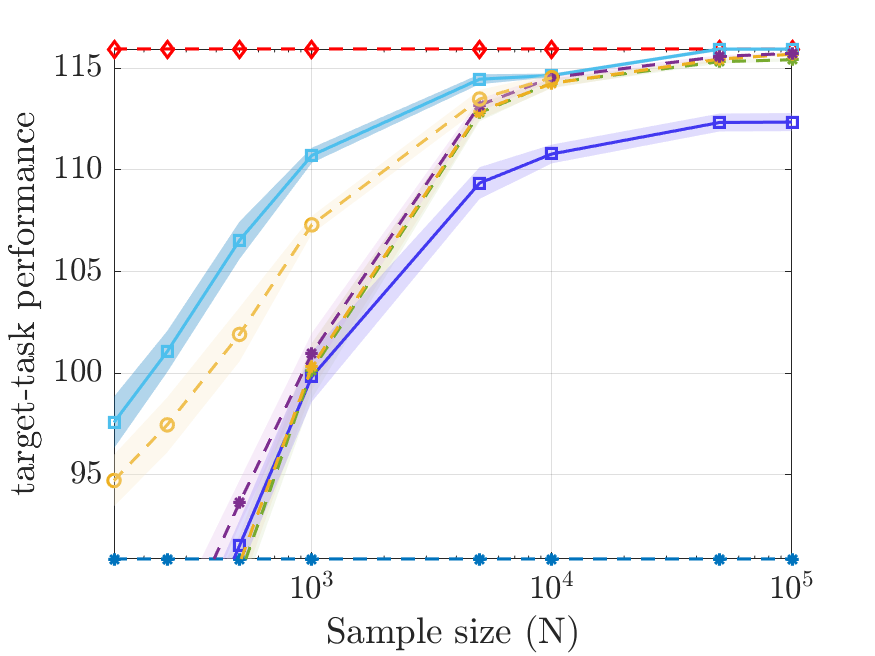}
    \caption{SOTA.}
    \label{fig: Acrobot_STOA}
  \end{subfigure}

  \includegraphics[width=.85\textwidth]{Figures/NR_legends.PNG}
 \caption{\small{Target domain performance for the non-robust setting averaged over $20$ runs as a function of sample size for Acrobot environment, and the corresponding $95\%$ confidence intervals.}}
  \label{fig: Acrobot_non_robust}
\end{figure}

\begin{figure}
  \centering
  \begin{subfigure}{0.35\textwidth}
    \includegraphics[width=\linewidth]{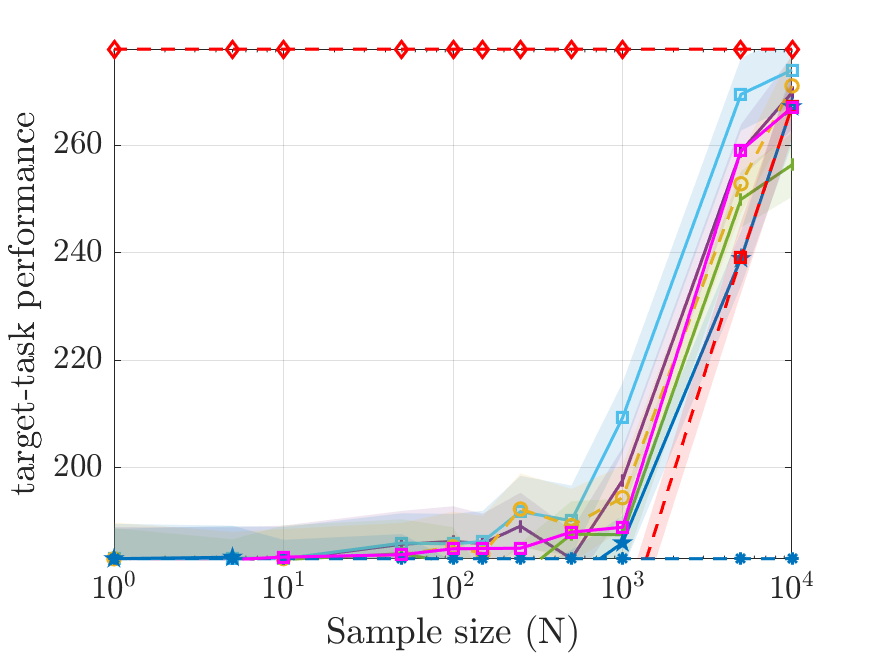}
    \caption{IBE comparison.}
    \label{fig: Pend_IBE}
  \end{subfigure}
  \hspace{-5mm}
  \begin{subfigure}{0.35\textwidth}
    \includegraphics[width=\linewidth]{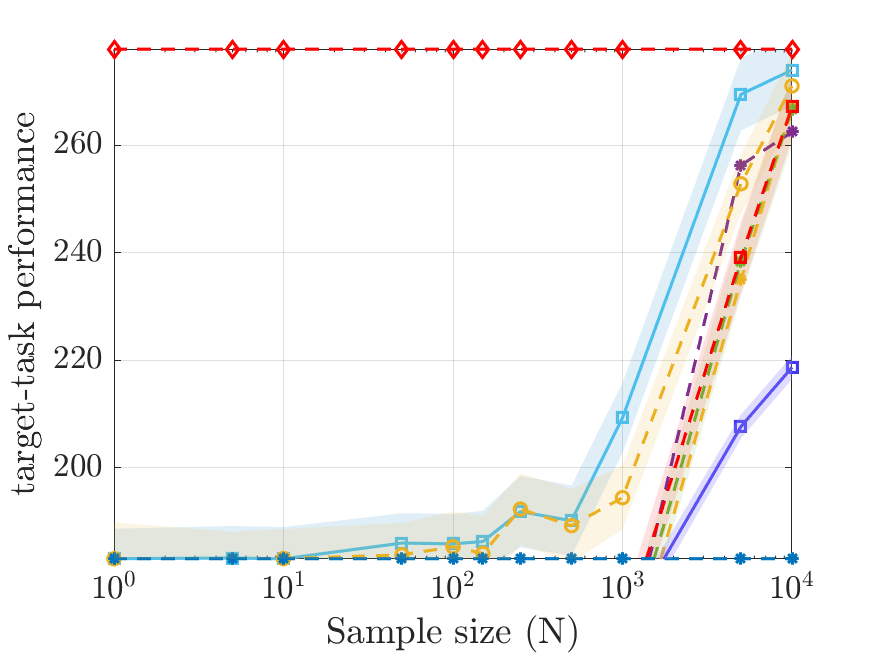}
    \caption{SOTA.}
    \label{fig: Pend_STOA}
  \end{subfigure}

  \includegraphics[width=.85\textwidth]{Figures/NR_legends.PNG}
 \caption{\small{Target domain performance for the non-robust setting averaged over $20$ runs as a function of sample size for Pendulum environment, and the corresponding $95\%$ confidence intervals.}}
  \label{fig: Pend_non_robust}
\end{figure}

\subsection{ Robust setting }\label{Appendix: robust_exp}

In this part, we present the IBE performance on the target domain for the robust scenario. Algorithm~\ref{alg:valueiteration} is used to estimate the IBE polices, where we choose \(R= 0.1\). All polices, including the IBE polices and the state-of-art (SOTA)  polices are evaluated on the target domain under uncertainty using Algorithm~\ref{alg:evaluatoin}, using the radius \(R=0.1\).
Figures \ref{fig: cliff_walking_robust}, \ref{fig: Taxi_robust},  \ref{fig: Cart_Pole_robust}, \ref{fig: Acrobot_robust}, and \ref{fig: Pend_robust} show the target task performance, namely, the average value function, i.e., $\frac{1}{|\mathcal{S}|}\sum_{s\in \mathcal{S}} V^{\pi}(s)$, as a function of the sample size $N$ for the robust scenario. Each figure shows the IBE with different side information (left panel) and a comparison with the SOTA methods (right panel).\par
\begin{figure}
  \centering
  \begin{subfigure}{0.35\textwidth}
    \includegraphics[width=\linewidth]{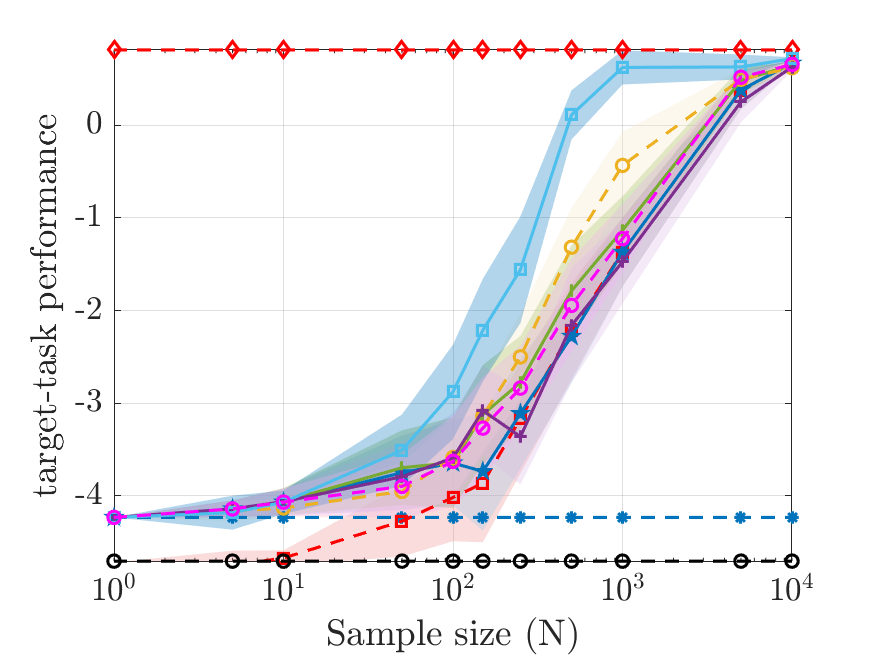}
    \caption{IBE comparison.}
   
  \end{subfigure}
  \hspace{-5mm}
  \begin{subfigure}{0.35\textwidth}
    \includegraphics[width=\linewidth]{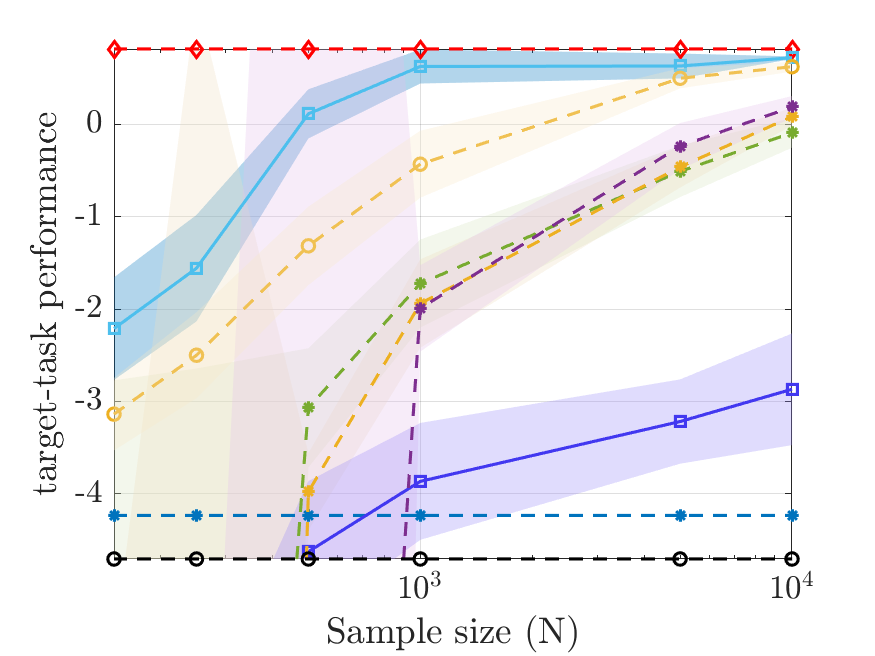}
    \caption{SOTA.}
    
  \end{subfigure}

  \includegraphics[width=.85\textwidth]{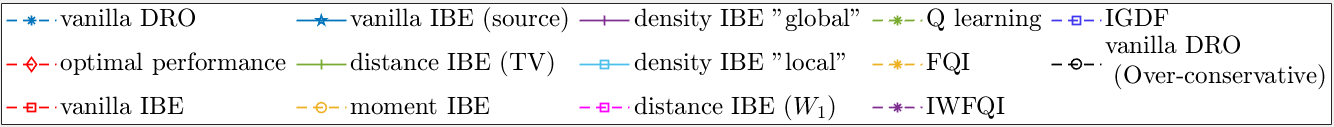}
 \caption{\small{Target domain performance for the robust setting averaged over $20$ runs as a function of sample size for Cliff Walking environment, and the corresponding $95\%$ confidence intervals.}}
  \label{fig: cliff_walking_robust}
\end{figure}

\begin{figure}
  \centering
  \begin{subfigure}{0.35\textwidth}
    \includegraphics[width=\linewidth]{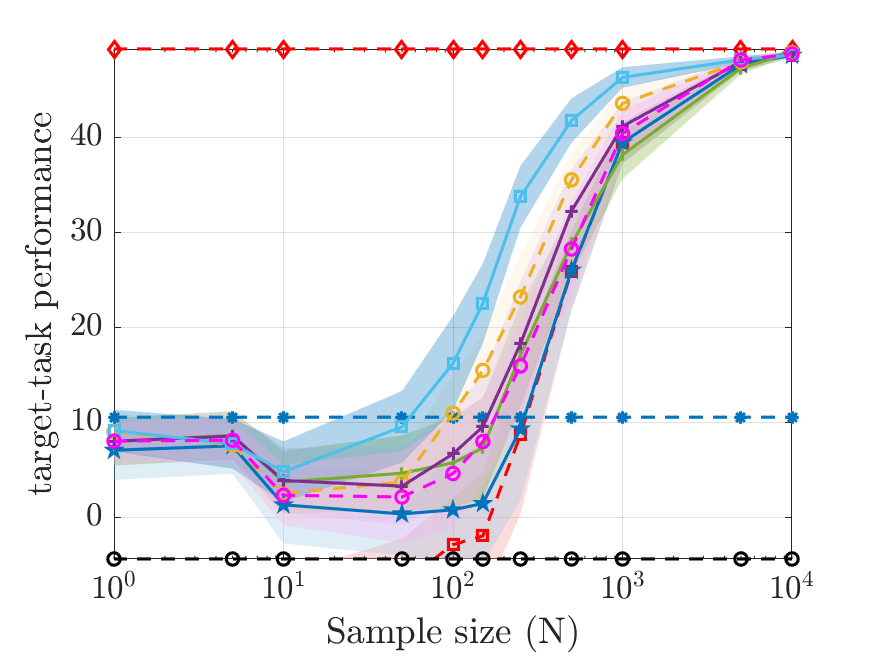}
    \caption{IBE comparison.}

  \end{subfigure}
  \hspace{-5mm}
  \begin{subfigure}{0.35\textwidth}
    \includegraphics[width=\linewidth]{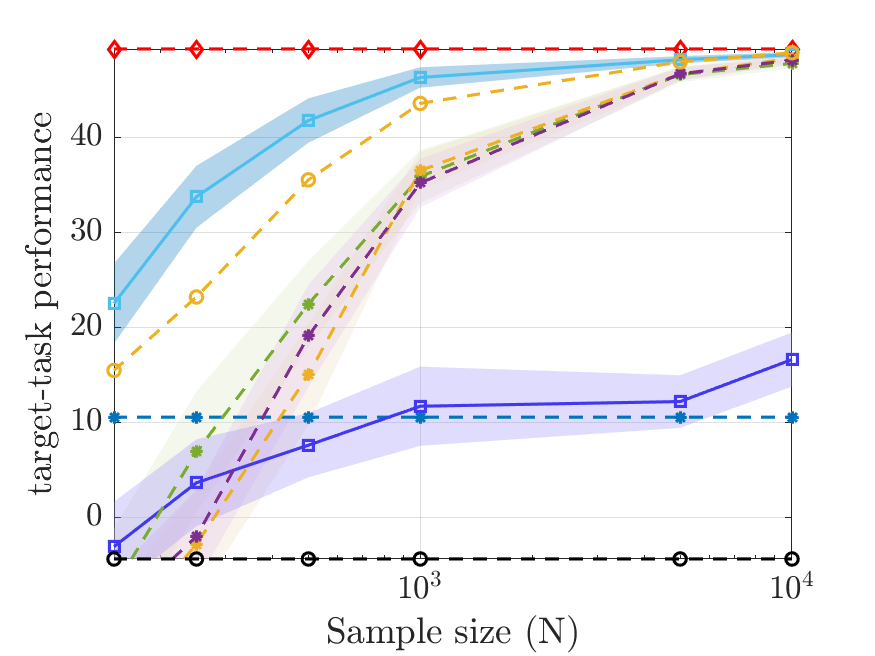}
    \caption{SOTA.}
  \end{subfigure}
 
      \includegraphics[width=.85\textwidth]{Figures/R_legends_new.PNG}
 \caption{\small{Target domain performance for the robust setting averaged over $20$ runs as a function of sample size for Taxi environment, and the corresponding $95\%$ confidence intervals.}}
  \label{fig: Taxi_robust}
\end{figure}

\begin{figure}
  \centering
  \begin{subfigure}{0.35\textwidth}
    \includegraphics[width=\linewidth]{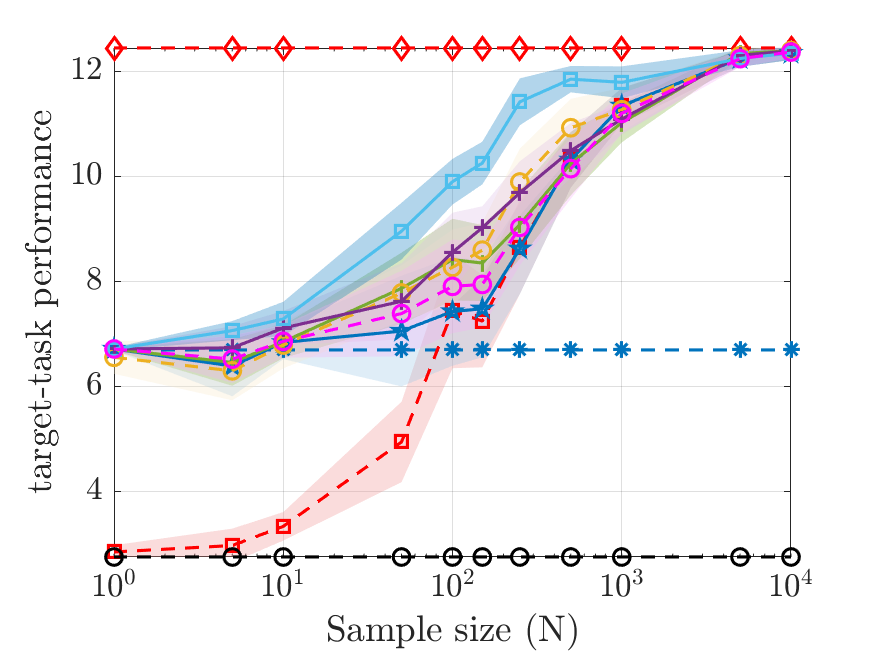}
    \caption{IBE comparison.}
  
  \end{subfigure}
  \hspace{-5mm}
  \begin{subfigure}{0.35\textwidth}
    \includegraphics[width=\linewidth]{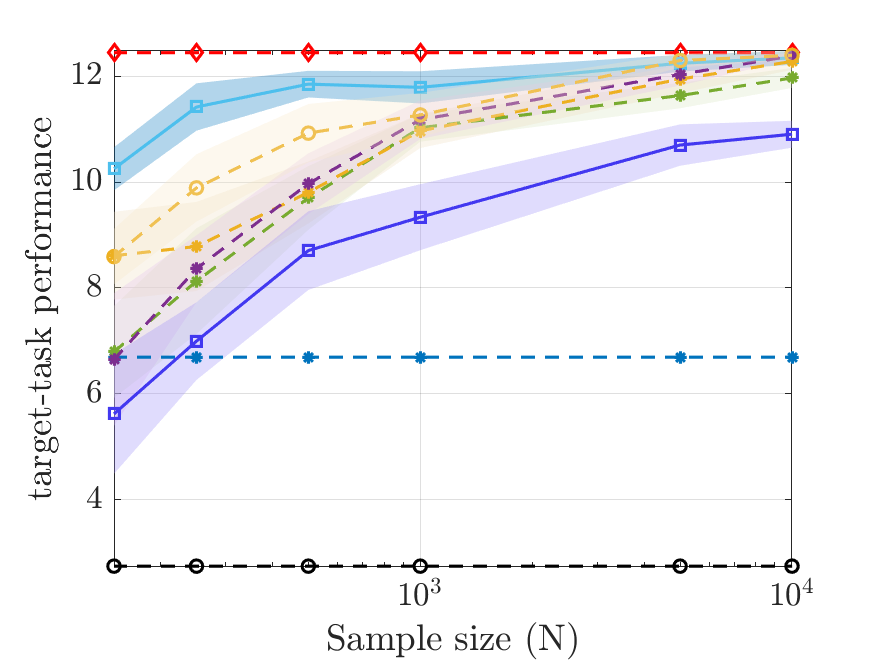}
    \caption{SOTA.}
  \end{subfigure}
 
  \includegraphics[width=.85\textwidth]{Figures/R_legends_new.PNG}
 \caption{\small{Target domain performance for the robust setting averaged over $20$ runs as a function of sample size for Frozen Lake environment, and the corresponding $95\%$ confidence intervals.}}
  \label{fig: Cart_Pole_robust}
\end{figure}

\begin{figure}
  \centering
  \begin{subfigure}{0.35\textwidth}
    \includegraphics[width=\linewidth]{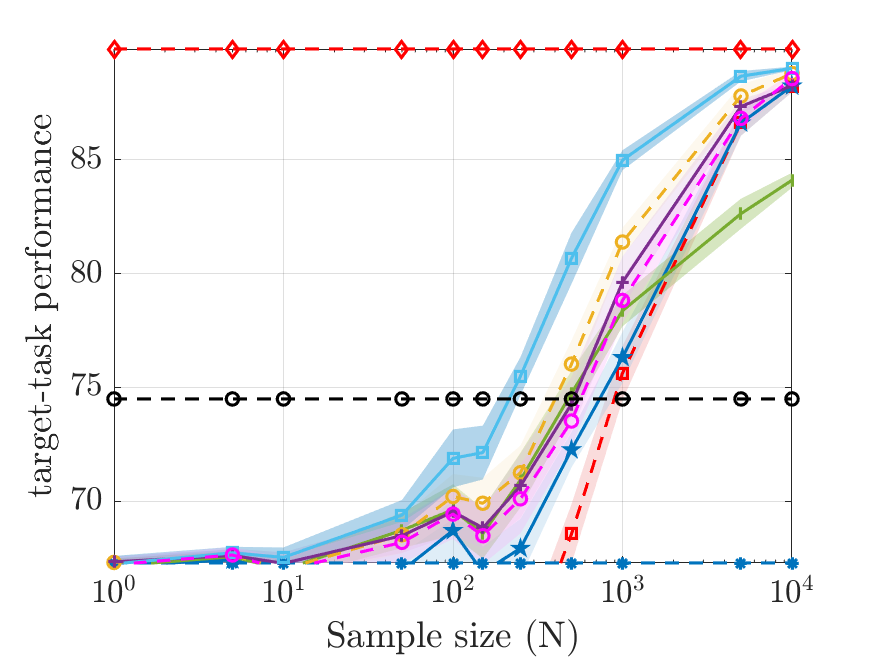}
    \caption{IBE comparison.}

  \end{subfigure}
  \hspace{-5mm}
  \begin{subfigure}{0.35\textwidth}
    \includegraphics[width=\linewidth]{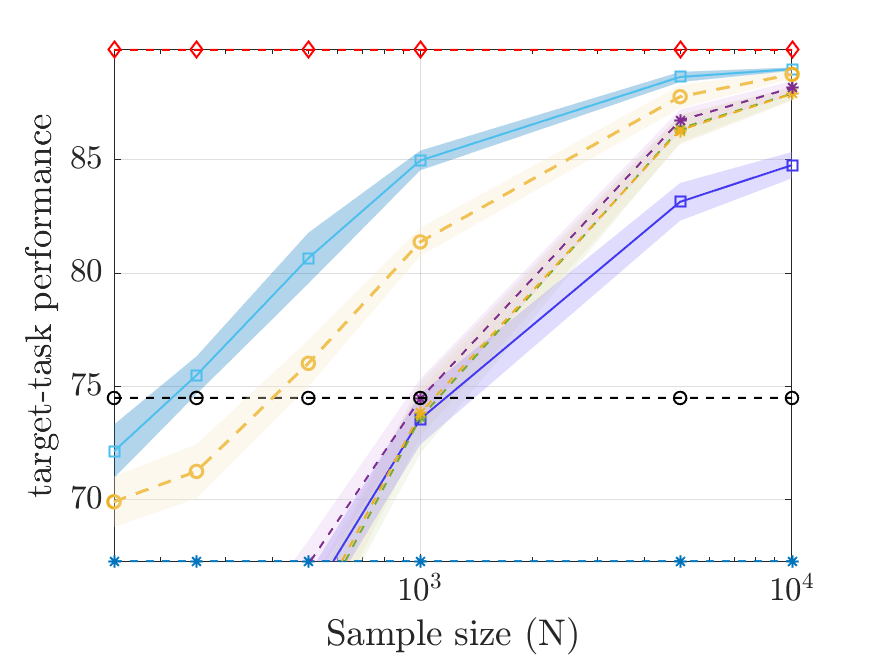}
    \caption{SOTA.}
  
  \end{subfigure}

  \includegraphics[width=.85\textwidth]{Figures/R_legends_new.PNG}
    
 \caption{\small{Target domain performance for the robust setting averaged over $20$ runs as a function of sample size  for Acrobot environment, and the corresponding $95\%$ confidence intervals.}}
  \label{fig: Acrobot_robust}
\end{figure}

\begin{figure}
  \centering
  \begin{subfigure}{0.35\textwidth}
    \includegraphics[width=\linewidth]{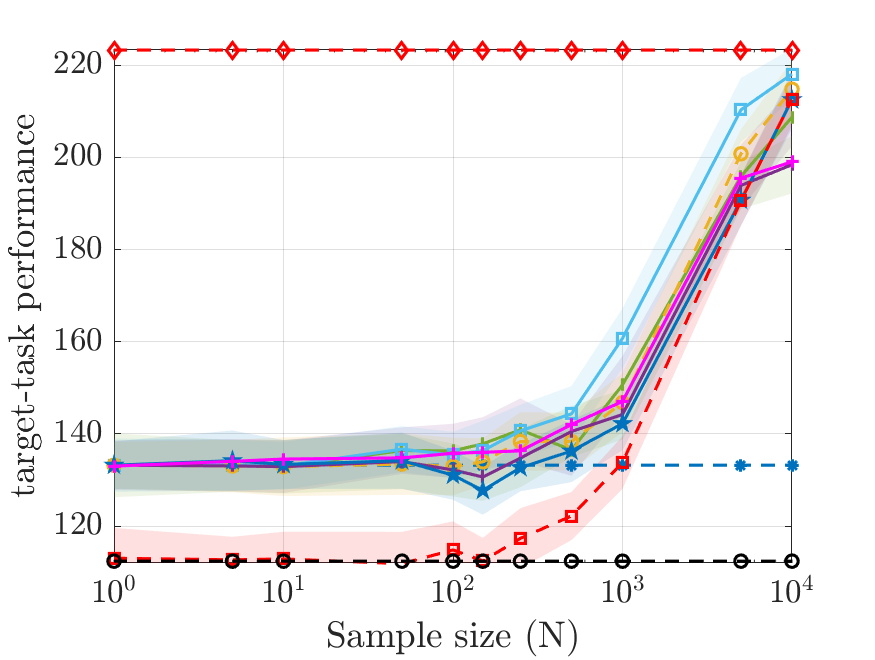}
    \caption{IBE comparison.}

  \end{subfigure}
  \hspace{-5mm}
  \begin{subfigure}{0.35\textwidth}
    \includegraphics[width=\linewidth]{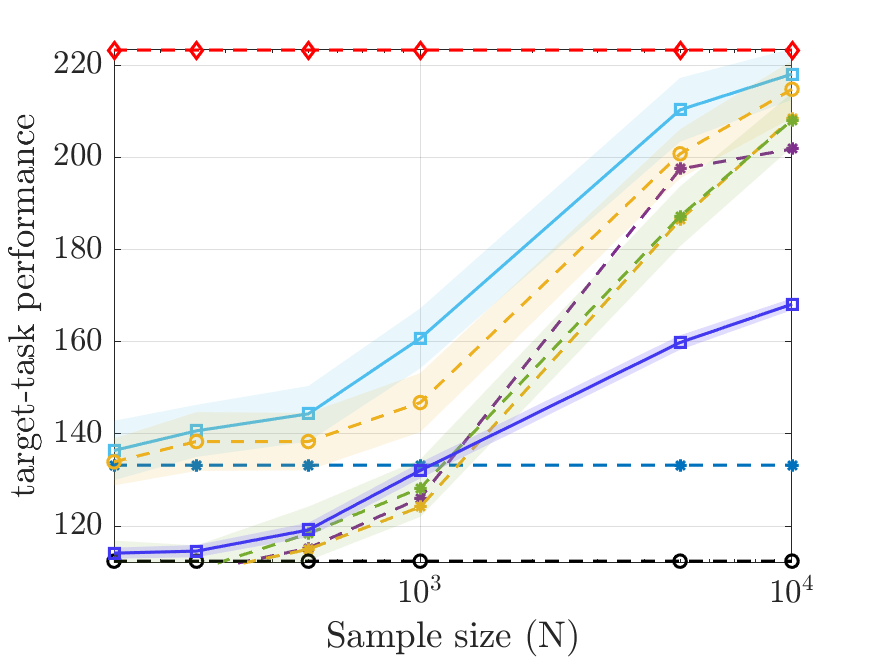}
    \caption{SOTA.}
  
  \end{subfigure}

  \includegraphics[width=.85\textwidth]{Figures/R_legends_new.PNG}
    
 \caption{\small{Target domain performance for the robust setting averaged over $20$ runs as a function of sample size  for Pendulum environment, and the corresponding $95\%$ confidence intervals.}}
  \label{fig: Pend_robust}
\end{figure}

\subsection{Evaluation error: Experimental verification}\label{Appendix: eval_error_exp}

In this experiment, we verify the theoretical results of the evaluation error bound presented in Theorem \ref{thm: eval error bound} as well as the convergence results of Corollary \ref{cor:consistency}. We consider the frozen lake environment, following the same settings and sampling procedure as in section \ref{section: numerical experiments}. The policy is estimated by Algorithm \ref{alg:valueiteration} using the transition kernels estimated using the approaches outlined in Table \ref{tab: estimation methods}. Then, each policy is evaluated on the target domain environment. We consider both non-robust and robust scenarios. Figure \ref{fig:ERROR_BOUND} and \ref{fig:ERROR_BOUND_R}  show the evaluation (EV) error (LHS of Equation \eqref{eqn: eval_err}) and the bound (RHS of Equation \eqref{eqn: eval_err}) as a function of the sample size for each estimation approach for the non-robust and robust settings, respectively.  As observed, the evaluation error is always upper bounded by the computed bound for all sample sizes. We also observe that the error and the bound decrease as the sample size increases which verifies the convergence results in Corollary \ref{cor:consistency}.\par
\begin{figure}
    \centering
\includegraphics[width=1\linewidth]{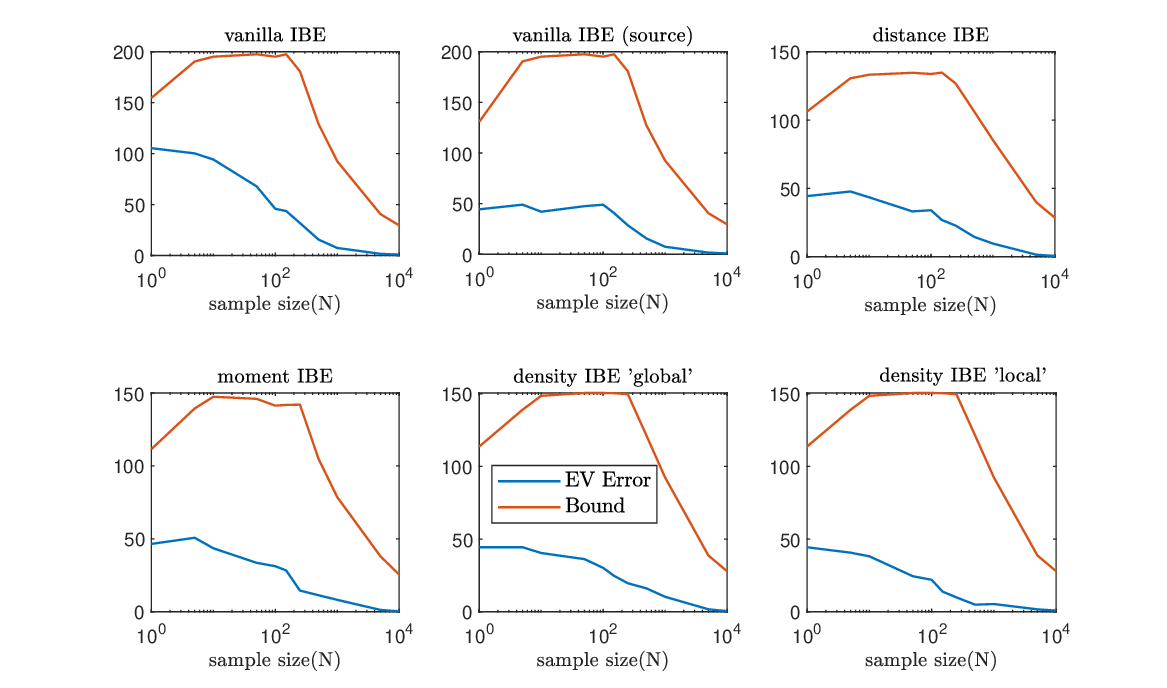}
    \caption{\small{The Evaluation (EV) Error and the bound for each MLE method as a function of the sample size for the non-robust scenario}}
    \label{fig:ERROR_BOUND}
\end{figure}

\begin{figure}
    \centering
\includegraphics[width=1\linewidth]{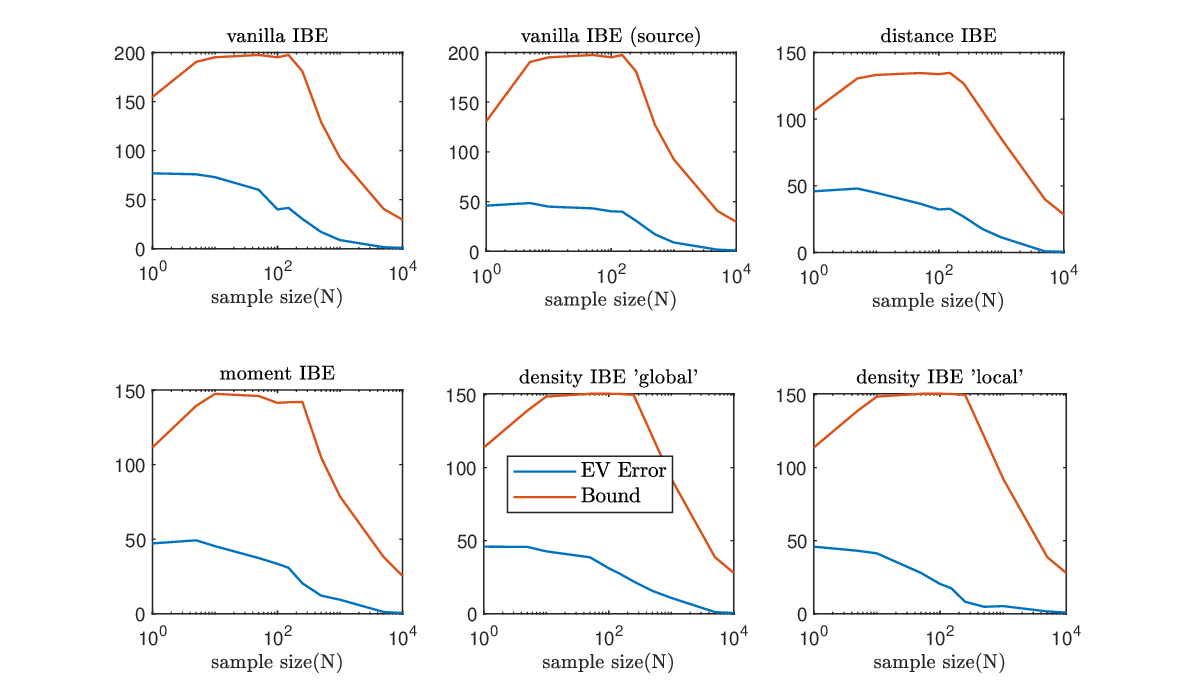}
    \caption{\small{The Evaluation (EV) Error and the bound for each IBE method as a function of the sample size for the robust setting.}}
    \label{fig:ERROR_BOUND_R}
\end{figure}

\subsection{The gain of prior information}\label{Appendix: cramer_rao_exp}
In this experiment, we investigate the information gain from prior information about the estimated kernels. In particular, given a $k\times k$ FIM $J(\kpt)$, we seek to measure the amount of information that $n$ i.i.d observations provide about the kernel $\kpt$, given some prior information. Therefore, we consider the following optimization problem:
\begin{align} \label{FIM prog}
    \min_{q\in \Delta(\mathcal{S})} \operatorname{trace}(J(q)) \quad \textup{s.t.}\quad \mathbb{E}^{q}[x^{j}] \leq c_j, 1\leq j\leq M.  
\end{align}
Here, we assume knowledge of the bounds \(c_j\) first $M$ moments. Note that we examine the worst-case scenario for the information the FIM carries by considering the minimum. We compare two scenarios: (i) `Without Constraints', in which we omit the expectation constraints, and (ii) `With Constraints', where the constraints are included. Figure \ref{fig:FIM} shows the optimal value of the $\operatorname{trace}(J(\kpn))$ (\textit{left}) and its inverse (\textit{right}) on a logarithmic scale. As shown, including the constraints significantly increases the amount of information about the estimator leading to a much better estimate. \par

\begin{figure}
    \centering

\includegraphics[width=0.7\linewidth]{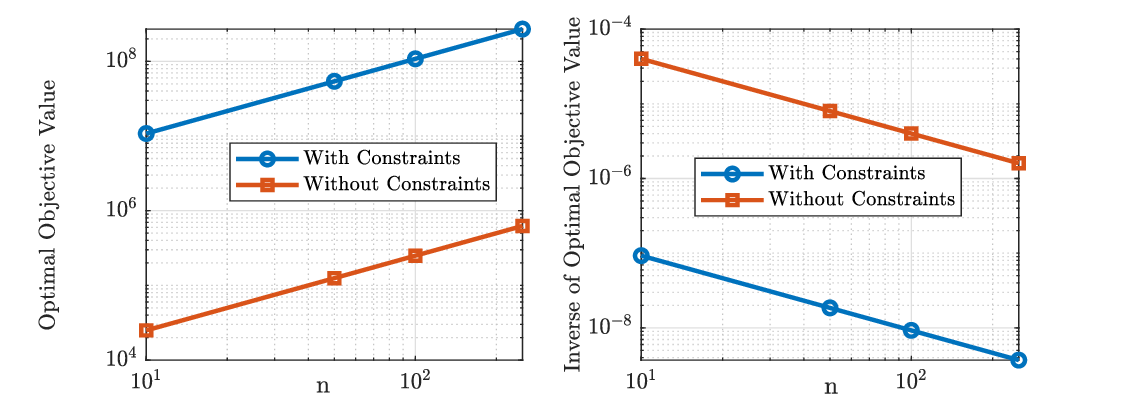}

    \caption{\small{The trace of the FIM as a function of the sample size $n$ (left) and its inverse (right).}}
    \label{fig:FIM}
\end{figure}

\end{document}